\documentclass{article}

    \PassOptionsToPackage{numbers, compress}{natbib}

    \usepackage[preprint]{neurips_2025}

\usepackage[utf8]{inputenc} 
\usepackage[T1]{fontenc}    
\usepackage{hyperref}       
\usepackage{url}            
\usepackage{booktabs}       
\usepackage{amsfonts}       
\usepackage{nicefrac}       
\usepackage{microtype}      
\usepackage{xcolor}         
\usepackage{amssymb,amsmath}
\usepackage[ruled]{algorithm2e} 

\usepackage{enumitem}

\title{Active Value Querying to Minimize Additive Error in Subadditive Set Function Learning}

\author{%
  Martin \v{C}ern\'{y}\thanks{Equal contribution (shared first authorship).} \\
  Charles University \\
  Prague, Czechia \\
  \texttt{cerny@kam.mff.cuni.cz} \\
  \And
  David Sychrovsk\'{y}\footnotemark[1] \\
  Charles University \\
  Prague, Czechia \\
  \texttt{sychrovsky@kam.mff.cuni.cz} \\
  \And
  Filip \'{U}radn\'{i}k\footnotemark[1] \\
  Charles University \\
  Prague, Czechia \\
  \texttt{uradnik@kam.mff.cuni.cz} \\
  \And
  Jakub \v{C}ern\'{y} \\
  Columbia University \\
  New York, USA \\
  \texttt{jakub.cerny@columbia.edu} \\
}

\usepackage{graphicx}
\usepackage{amsfonts}  
\usepackage{amsmath,thm-restate}  
\usepackage{todonotes}  
\usepackage{mathtools} 
\usepackage{xcolor}

\usepackage{csquotes}
\usepackage{subcaption}
\usepackage{cleveref}
\makeatletter
\AddToHook{cmd/appendix/before}{\def\cref@chapter@alias{appendix}}
\AddToHook{cmd/appendix/before}{\def\cref@section@alias{appendix}}
\AddToHook{cmd/appendix/before}{\def\cref@subsection@alias{appendix}}
\makeatother

\usepackage{thmtools}

\newcommand{\argmin}{\operatornamewithlimits{arg\,min}}
\newcommand{\argmax}{\operatornamewithlimits{arg\,max}}

\makeatletter
\newcommand{\pushright}[1]{\ifmeasuring@#1\else\omit$\displaystyle#1$\ignorespaces\fi}
\makeatother

\newcommand{\BibTeX}{\rm B\kern-.05em{\sc i\kern-.025em b}\kern-.08em\TeX}

\newcommand{\R}{\ensuremath{\mathbb{R}}}

\newcommand{\K}{\ensuremath{\mathcal{K}}}
\renewcommand{\L}{\ensuremath{\mathcal{L}}}
\renewcommand{\ss}{\ensuremath{\mathcal{S}}}

\newcommand{\valueDistribution}{\mathcal{F}}
\newcommand{\gap}{\ensuremath{\Delta}}

\newcommand{\norm}[1]{\left\lVert#1\right\rVert}

\renewcommand{\S}{\ensuremath{\mathbb{S}^n}}
\renewcommand{\SS}{\ensuremath{\mathbb{SS}^n}}

\newcommand{\SAM}{\ensuremath{\mathbb{SAM}^n}}
\newcommand{\XOS}{\ensuremath{\mathbb{XOS}^n}}
\newcommand{\SCMM}{\ensuremath{\mathbb{SCMM}^n}}

\newcommand{\C}{\ensuremath{\mathbb{C}^n}}
\newcommand{\CA}{\ensuremath{\mathbb{CA}^n}}

\def\absolute#1{\left\lvert #1 \right\rvert }

\usepackage{amsthm}

\theoremstyle{plain}
\newtheorem{definition}{Definition}

\newtheorem{lemma}{Lemma}

\newtheorem{observation}{Observation}

\newtheorem*{acks*}{Acknowledgements}

\DeclareMathOperator*{\supp}{supp}
\DeclareMathOperator*{\poly}{poly}

\usepackage{tikz,pgfplots}

\pgfplotsset{compat=1.18}
\usepgfplotslibrary{fillbetween}
\usetikzlibrary{positioning}
\usepackage{pgffor}
\usepgfplotslibrary{groupplots}
\usepgfplotslibrary{colormaps}
\pgfplotsset{ytick style={draw=none}}
\pgfplotsset{ymajorgrids=true}
\pgfplotsset{xmajorgrids=true}

\pgfplotsset{colormap/thermal} 

\pgfplotscolormapaccess[1:4]{1}{viridis}
\definecolor{expected_greedy}{rgb}{\pgfmathresult}
\pgfplotscolormapaccess[1:4]{2}{viridis}
\definecolor{expected_best_states}{rgb}{\pgfmathresult}
\pgfplotscolormapaccess[1:4]{3}{viridis}
\definecolor{eval_ln_lin}{rgb}{\pgfmathresult}
\definecolor{eval}{rgb}{\pgfmathresult}
\definecolor{eval_ln}{rgb}{\pgfmathresult}
\pgfplotscolormapaccess[1:4]{4}{viridis}
\definecolor{random_eval}{rgb}{\pgfmathresult}

\tikzset{random_eval_style/.style={solid}}
\tikzset{best_states_style/.style={solid}}
\tikzset{solve_greedy_style/.style={densely dashed}}
\tikzset{eval_style/.style={densely dotted}}
\tikzset{eval_ln_lin_style/.style={densely dotted}}
\tikzset{eval_ln_style/.style={densely dotted}}
\tikzset{expected_best_states_style/.style={densely dashdotted}}
\tikzset{expected_greedy_style/.style={densely dashed}}

\newcommand\exploitabilityplot[3]{
	\addplot [#2_style, style=thick, color=#2, mark=none] table [col sep=space, x=step, y=#2] {#1};
	\addlegendentry{#3}
	\addplot [forget plot, name path=#2_upper,draw=none, mark=none] table[x=step,y expr=\thisrow{#2}+\thisrow{#2_err}] {#1};
	\addplot [forget plot, name path=#2_lower,draw=none, mark=none] table[x=step,y expr=\thisrow{#2}-\thisrow{#2_err}] {#1};
	\addplot [forget plot, fill=#2, opacity=0.3] fill between[of=#2_upper and #2_lower];
}

\def\exploitabilityplotwidth{0.3\textwidth}
\def\exploitabilityplotheight{0.3\textwidth}
\def\exploitabilityplotsep{1.8cm}
\tikzset{font=\scriptsize}
\tikzset{/pgfplots/legend style={very thin, font=\scriptsize,/tikz/every even column/.append style={column sep=0.2cm}}}
\tikzset{/pgfplots/legend image post style={xscale=0.5}}

\def\offlinegreedy{\textsc{Offline Greedy}}
\def\offlineoptimal{\textsc{Offline Optimal}}
\def\oraclegreedy{\textsc{Oracle Greedy}}
\def\oracleoptimal{\textsc{Oracle Optimal}}
\def\random{\textsc{Random}}
\def\ppo{\textsc{PPO}}

\newcommand\covg[1][n]{\texttt{sam-covg}(\ensuremath{#1})}
\newcommand\convex[1][n]{\texttt{submod-neg}(\ensuremath{#1})}
\newcommand\kbudget[1][n]{\texttt{$k$-budget}(\ensuremath{#1})}
\newcommand\xos[1][n]{\texttt{xos-6}(\ensuremath{#1})}

\def\pr#1{\Pr\! \left[ #1 \right]}
\def\viable{\mathcal{V}}

\usepackage{tikz}
\usetikzlibrary{decorations.pathreplacing,calc}

\begin{document}

\maketitle

\begin{abstract}

\noindent Subadditive set functions play a pivotal role in computational economics (especially in combinatorial auctions), combinatorial optimization or artificial intelligence applications such as interpretable machine learning. However, specifying a set function requires assigning values to an exponentially large number of subsets in general, a task that is often resource-intensive in practice, particularly when the values derive from external sources such as retraining of machine learning models. A~simple omission of certain values introduces ambiguity that becomes even more significant when the incomplete set function has to be further optimized over. Motivated by the well-known result about inapproximability of subadditive functions using deterministic value queries with respect to a multiplicative error, we study a problem of approximating an unknown subadditive (or a subclass of thereof) set function with respect to an additive error -- i. e., we aim to efficiently close the distance between minimal and maximal completions. Our contributions are threefold: (i) a thorough exploration of minimal and maximal completions of different classes of set functions with missing values and an analysis of their resulting distance; (ii) the development of methods to minimize this distance over classes of set functions with a known prior, achieved by disclosing values of additional subsets in both offline and online manner; and (iii) empirical demonstrations of the algorithms' performance in practical scenarios.

\end{abstract}

\section{Introduction} \label{sec: intro}

Set function optimization comes into play whenever we need to optimize over subsets of elements from a larger ground set, where adding more inputs to a fixed subset has (possibly) non-linear additional benefits depending on the other elements. Its applications span diverse fields, including combinatorial auctions~\cite{bhawalkar2011welfare,dobzinski2005approximation,feige2006maximizing}, supply chain management~\cite{nagarajan2008game}, communication networks~\cite{saad2009coalitional}, logistics and resource allocation~\cite{lozano2013cooperative}, or environmental agreements~\cite{finus2008game}. 
For efficient optimization, we often rely on an oracle access to evaluate function values. However, describing a set function for $n$ items in general requires $2^n$ values -- one for each possible subset. It is natural to ask whether a reasonable approximation of such functions can be obtained with significantly less information while allowing for some controlled loss in accuracy, particularly when the function exhibits structural properties often required by the application, such as subadditivity.

Subadditive functions appear in a variety of applications, particularly in economics, where they capture the absence of complementarities between items. For example, a buyer's valuation is subadditive if the value of acquiring two sets of items together does not exceed the sum of their separate values. Such valuation models are common in combinatorial auctions and have been studied extensively~\cite{feige2006maximizing,dobzinski2005approximation}. Subadditive functions represent the most general class in a hierarchy of complement-free valuations, which includes more restrictive classes such as submodular functions~\cite{lehmann2001combinatorial}. Beyond economics, subadditivity also arises in areas such as risk modeling, where it reflects the principle that diversification does not increase overall risk~\cite{Artzner1999}, or algorithm design, where subadditive cost functions are used in facility location and network design problems~\cite{Shmoys1997}.

The problem of finding efficient (subadditive) set function approximations, typically using a small number of set value queries, has been studied from various angles. Perhaps the most developed is the literature on Probably Mostly Approximately Correct (PMAC) learning, a PAC-style passive learning model from existing datasets. Initial work in this area showed strong impossibility results for learning general submodular functions under PMAC guarantees~\cite{balcan2011learning,balcan2010submodular}. In response, subsequent research has focused on identifying subclasses of functions, such as coverage, self-bounding, and XOS functions, that allow more favorable learning outcomes \cite{FeldmanKothari2014,FeldmanEtAl2017,balcan2012}. Some approaches leverage structural properties such as bounded curvature or decision tree representations \cite{feldman2013representation}, while others exploit connections to Fourier analysis of Boolean functions \cite{feldman2013optimal}. An approach more closely related to ours, due to its active nature, comes from the literature on preference elicitation. In this setting, the learner actively queries an agent to uncover its valuation or preference function, often with the goal of making efficient decisions in combinatorial problems. This, typically, does not require to recover the function for every set. Early work showed that, under certain assumptions, preference elicitation can be done in polynomial time using value or comparison queries~\cite{HudsonSandholm2003,ZinkevichEtAl2003}. Other works have focused on the query complexity of such models under additional structural assumptions, such as substitutability or submodularity~\cite{lahaie2004applying}, or considered minimizing the number of queries needed to approximate the preference function accurately~\cite{blum2004near}. Besides value queries, some works have also considered demand queries as a more powerful alternative~\cite{zhang2020learning}. In contrast to preference elicitation, another line of work aims to approximate valuation functions everywhere, providing guarantees on the accuracy of the approximation for all subsets. Some papers develop succinct representations, such as \(\alpha\)-sketches~\cite{goemans2009approximating,badanidiyuru2012sketching}. Others study notions of approximate modularity to measure how close a function is to modular and derive corresponding approximation results \cite{chierichetti2015approximate,feige2017approximate}. 

For the majority of algorithmic works above, the number of queries being ``small'' is defined by asymptotics in terms of the size of the ground set $n$. For our contributions, we adopt colloquial, but often times more practically useful meaning of ``small'', meaning ``small enough'' (e.g., say, in single or lower double digits). This is motivated by a fact that in practice, acquiring even a single value for a single subset can be a resource-intensive endeavor. Take, for instance, the realm of machine learning, where determining the value of a feature subset in the already well established explainable-AI algorithm SHAP~\cite{NIPS2017_7062} may correspond to retraining a significant portion of the model, consuming time and computational resources. Similarly, in the corporate world, estimating an employee's contribution to collective performance may facilitate their fair evaluation~\cite{murphy1995understanding}, but obtaining the value of a \enquote{subset} may involve a process of rearranging teams of employees.


Secondly, a significant amount of work approximates a set function within a multiplicative factor. Multiplicative approximations are often preferred when the identified set function is later used in approximation algorithms with relative performance guarantees. However, \citet{badanidiyuru2012sketching} prove the following negative result: if a deterministic algorithm can identify an approximate set function of every subadditive function using $\mathcal O(\poly(n))$ value queries, then the multiplicative factor is of order $\Omega(n^{1-\epsilon})$ for every $\epsilon>0$\footnote{We remark that our approach, minimizing additive error, gives similar multiplicative bounds, see \Cref{sec: experiments}.}. Instead, in this paper we address the following question: 
\begin{center}
\textit{Given a budget of $k$ value queries, which subsets should we query to find a approximate set function with minimal total additive error across all sets}?    
\end{center}
Additive error bounds can be useful in different kinds of settings. 
Consider, for instance, the SHAP framework, where one seeks to estimate each feature's contribution to a set function, i.e., how “valuable” a feature is. More broadly, in cooperative game theory, players lacking precise information about subset (coalition) values may overestimate their individual contributions, leading to unrealistic demands that can exceed the total value of the grand coalition~\cite{uradnik2024reducing}. This mismatch is seen when companies overprice data or employees seek inflated wages. Such gaps reflect a divergence between perceived and feasible outcomes. In this work, we aim to minimize this divergence, defined as the distance between two extremal completions of the partial set function.

To do so, we assume an external principal, such as a manager or ML engineer, who controls which subset values are revealed, subject to a fixed query budget. Each query is assumed to have equal cost, making the principal’s sole goal to reduce divergence as efficiently as possible. While we do not specify how to handle residual divergence, we assume lower divergence is generally preferred. The main contribution of this paper is the study of tight completions for various subclasses of subadditive functions, which allow the principal to reduce divergence more effectively. Leveraging these tighter completions, we develop algorithms for selecting which subsets to query, in both offline and online settings. Although the subproblems involved are exponential in nature and computationally intensive, their decomposable structure makes them amenable to parallelization. As discussed earlier, in many practical applications, reducing the number of queries may be more valuable than limiting (relatively cheap) computational cost. Finally, our experimental results demonstrate the typical reduction in divergence that can be achieved and highlight the benefits of informed subset selection compared to baseline naive strategies.

\section{Preliminaries and Problem Formulation}

Let us begin by defining basic notions we operate with.

\begin{definition}
	\label{def: set function}
	A \emph{set function} $f\colon 2^N \to \mathbb{R}$ assigns values to subsets of the ground set $N$.
We say $f$ is \emph{subadditive} if $\forall S, T \subseteq N, S\cap T=\emptyset$, $f(S) + f(T) \ge f(S\cup T)$. We denote the set of subadditive functions by $\S$.
\end{definition}

To model a scenario when acquiring (or storing) all $2^n$ values proves to be cost-prohibitive, and only selected values are known, we introduce an \textit{incomplete set function} $(f, \K)$, where $\K \subseteq 2^N$ denotes the set of subsets where the values are known. Expanding the set $\K$ then models the acquisition of new information about the initially unknown (yet well-defined) values. The set $\K$ can be hence seen as a \enquote{masking set}, acting as a filter applied to a complete function.  
Assuming additional properties of the underlying set function $f$, i.e., $f$ belonging to a class $\C$ of set functions on ground set $n$, one can impose constraints on values of $S \notin \K$, even in the absence of the exact knowledge. For reasons outlined in the introduction, we focus on subclasses of subadditive set functions $\S$. 
To bound set $\C$, and thus also all values of $S \notin \K$, 
it is necessary to know at least the {\it minimal information} $\K_0$, defined as $\K_0 \coloneqq \{\emptyset,N\} \cup \{\{i\} \mid i \in N\}$.

\begin{definition}
	Let $(f, \K)$, $\K_0\subseteq \K$ be an incomplete set function. Then $g$ is a \emph{$\C$-extension} of $(f,\K)$ if $g \in \C$ and $ \forall S\in \K:
		f(S) = g(S)$.
We say $(f,\K)$ is \emph{$\C$-extendable} if it has an $\C$-extension and we denote the set of $\C$-extensions by $\C(f, \K)$. 
\end{definition}
Sets of $\C$ extensions studied in this text form a compact subset of $\R^{2^n}$, therefore, each unknown value of any $S$ can be bounded by a closed real interval. The bounds of these intervals can be captured by the \emph{lower/upper} completion functions.
\begin{definition}
    Let $(f,\K)$, $\K_0 \subseteq \K$ be an $\C$-extendable incomplete set function. We call $\underline{f}_\K$, $\overline{f}_\K$ \emph{$\C$-lower} and \emph{$\C$-upper functions of $(f,\K)$} if for every $g \in \C(f,\K)$, and for every $S \subseteq N$, it holds $\underline{f}_\K(S) \leq g(S) \leq \overline{f}_\K(S)$.
    Further, these functions are \emph{$\C$-tight}, if additionally  $\forall S \notin \K$, there are $\C$-extensions $\underline{g}, \overline{g}$ such that $\underline{g}(S) = \underline{f}_\K(S)$ and $\overline{g}(S) = \overline{f}_\K(S)$.
\end{definition}
Different classes $\C$ might have different tight upper/lower functions and will be of focus of the next section. We write $\underline{f}_\K[\C]$, $\overline{f}_\K[\C]$ to stress this relation to $\C$ when needed. 
The \enquote{size} of the set of extensions $\C(f, \K)$ represents the amount of uncertainty arising from only knowing values of $f$ in $\K$. 
We call the distance of the lower and the upper functions the \textit{divergence}, characterizing on the amount of uncertainty in an additive sense.

\begin{definition}
	\label{def: exploitability}
	Let $f\in\C$ and $\|\cdot\|$ be a set function norm. The 
    \emph{$\C$-divergence} of $f$ and $\K$ induced by $\|\cdot\|$ is a function $\gap_f \colon 2^{2^N \setminus \K_0} \to \mathbb{R}$ defined as $\gap_f(\C,\K) \coloneqq \left\|\overline{f}_{\K_0 \cup \K}[\C] - \underline{f}_{\K_0 \cup \K}[\C]\right\|$.
\end{definition}
We use $\Delta_f(\K)$ when it is obvious from context we talk about a specific $\C$-divergence.
It follows trivially from the properties of norms that the divergence is non-negative.
Furthermore, it is zero if and only if $\forall S\subseteq N: \overline{f}_\K(S) = \underline{f}_\K(S)$, or equivalently, when there is just a single extension.

Since 
the set functions on a ground set $N$ can be regarded as elements of a vector space $\mathbb{R}^{2^n}$, we may use any vector norm as a divergence inducing norm. 
In the reminder of this text, we shall focus on divergences induced by absolute norms, i.e. those, where the norm of $x$ is equal to the norm of $|x|$. These include all $l_p$-norms as well as many others. Since the divergence is non-negative, this restriction is without loss of generality. In addition, it has the following properties:


\begin{restatable}[]{proposition}{gapInK}
\label{prop: gap in K}
Let $f\in\S$, $\C \subseteq \S$ and $(\alpha f + \beta)(S) \coloneqq \alpha f(S) + \sum_{i \in S} \beta_i$. The $\C$-divergence is
\begin{enumerate}[itemsep=-3pt, topsep=-3pt]
	\item monotonically non-increasing, i.e., $\gap_f(\K) \ge \gap_f(\mathcal{L})$
      for $\K \subseteq \mathcal{L} \subseteq 2^{N} \setminus \K_0$,
	\item subadditive, i.e., $ \gap_f(\K) + \gap_f(\mathcal{L}) \geq \gap_f(\K \cup \mathcal{L})$
      for $\K,\mathcal{L} \subseteq 2^N\setminus \K_0$ such that $\K \cap \mathcal{L} = \emptyset$,
    \item normalizable, i.e., $\gap_{\alpha f + \beta}(\K) = \alpha \cdot \gap_f(\K)$
    for $\alpha > 0$, $\beta_i \in \mathbb{R}$ for $i \in N$.
\end{enumerate}
\end{restatable}


Large divergence indicates high uncertainty in missing values. As more is known about the underlying function, the divergence shrinks, becoming zero when there is a single $\C$-extension. Our goal is to reduce divergence as much as possible using a limited number of value queries. To formalize this, we assume the presence of a \textit{principal}, tasked with selecting which subsets of $N$ to query in order to minimize divergence. The principal has domain expertise, modeled as a prior distribution over possible functions. For instance, in medicine, a doctor may choose drug combinations based on clinical experience; similarly, a machine learning engineer may prioritize features using past knowledge. We assume each instance is drawn from a known prior $\valueDistribution$\footnote{Assuming the true function lies in $\valueDistribution$ ensures divergence behaves well: each new value narrows the set of consistent extensions, and the bounds converge to the true function. If not, the bounds may become inconsistent (e.g., lower exceeding upper), rendering divergence ill-defined.}. In short, the principal knows the prior but not the specific instance. Subset selection can then proceed in one of two modes: 

\begin{definition}[Value Querying Problems]\label{def: online principal problem}
	Let $t\in\mathbb{N}$, $\valueDistribution$, $\supp \valueDistribution\subseteq\C$ be a distribution on a class of subadditive set functions and $\gap_f$ be a $\C$-divergence.
	Then $\K^*_t\subseteq 2^N \setminus \K_0$ is a solution of
    \begin{itemize}[itemsep=-3pt, topsep=-3pt]
    \item \textit{offline principal's problem} of size $t$ if
    $
		\K_t^* \in
		\argmin_{\K\subseteq 2^N\setminus \K_0, |\K|= t}\left\{
		\mathop{\mathbb{E}}_{f\sim \valueDistribution}\left[\gap_f(\K)\right]\right\}
    $, or
        \item \textit{online principal's problem} of size $t$ if
    $
		\K^*_t \in
		\argmin_{\K_t}\left\{
		\mathop{\mathbb{E}}_{f\sim \valueDistribution}\left[\gap_f(\K_t)\right]\right\},
    $
	where $\K_t = \K_{t-1} \cup \left\{ S_t \right\} $, $S_t \in 2^N \setminus \K_{t-1}$ and $\K_{t-1}$ is a solution to the problem of size $ t-1$.
    \end{itemize}
\end{definition}

\section{Upper and Lower Completions for Classes of Subadditive Functions}\label{sec:bounds}

In this section, we derive upper and lower completions (or bounds) for several key subclasses of subadditive set functions. Tighter bounds yield smaller divergences and thus better approximation guarantees. We focus on three major classes: subadditive monotone functions; fractionally subadditive (XOS) functions, central to assignment and auction settings; and a subclass of SCMM functions with applications in machine learning \cite{Bilmes2017, Stobbe2010}. For subadditive monotone functions, we provide a new characterization for non-negative functions, clarifying their relation to XOS. We also derive tighter completions, along with a more computationally efficient (but slightly looser) alternative.

All derived bounds form a hierarchy of increasingly tight completions, based on progressively stronger assumptions: from subadditive (\S), to subadditive monotone (\SAM), to XOS, and finally SCMM (\SCMM). Within SCMM, we give bounds for concave additive (\CA) and symmetric submodular (\SS) functions. We have $\SS \subsetneq \CA \subsetneq \SCMM \subsetneq \XOS \subsetneq \SAM \subsetneq \S$. Since all bounds are tight, each step gives a tighter enclosure of completions: 
\begin{equation*} \gap_f(\SS,\K) \leq \gap_f(\CA,\K) \leq \gap_f(\XOS,\K) \leq \gap_f(\SAM,\K) \leq \gap_f(\S,\K).
\end{equation*} 
These inequalities may not always be strict, but examples show that using a bound from a superclass can lead to significantly worse results. Throughout, we assume $\K_0 \subseteq \K$ for each $(f, \K)$, and that $(f, \K)$ is $\C$-extendable when constructing $\C$-upper and $\C$-lower functions, for our results to hold.

\subsection{Subadditive functions}
For the class of subadditive functions $\S$, following results by~\citet{Masuya2016}, the $\S$-tight upper function $\overline{f}_{\K}$ of $(f,\K)$ is a (complete) set function given by
\begin{equation}
    \label{eq:s-upper-function}
    \overline{f}_\K[\S](S) \coloneqq \min\nolimits_{\substack{S_1,\dots,S_k \in \K,\,
            \bigcup_i S_i = S,\,
            S_i \cap S_j = \emptyset}}
    \sum\nolimits_{i=1}^k f(S_i),
\end{equation}
and the \S-tight lower function of $(f,\K)$ is
$
    \underline{f}_{\K}[\S](S) \coloneqq \max_{T \in \K: S \subseteq T} \left( f(T) - \overline{f}_\K(T \setminus S)\right).
$

\begin{restatable}{proposition}{sBounds}
\label{prop:S-bounds}
Functions $\overline{f}_\K[\S]$ and $\underline{f}_\K[\S]$  are \S-tight upper and lower functions.
\end{restatable}

\subsection{Subadditive monotone functions}
Subadditive monotone set functions are part of the hierarchy of complement-free functions and are among the most widely studied classes of set functions. The monotonicity property captures the straightforward idea that adding more elements to a set never lowers its value. This is crucial in models where extra resources, agents, or items are never a disadvantage—such as in production or cooperative scenarios.
\begin{definition}[Subadditive monotone function]
Set function $f \colon 2^N \to \R$ is \emph{subadditive monotone} if it is subadditive and further $f(S) \leq f(S \cup i)$ for $S \subseteq N$ and $i \in N \setminus S$.
We denote such functions on a ground set of size $n$ by $\SAM$.
\end{definition}



The \SAM-upper tight function has the following form:
\begin{equation}\label{eq:sam-upper-function}
    \overline{f}_{\K}[\SAM](S) \coloneqq \min\nolimits_{\substack{S_1,\dots,S_k \in \K,\,
            \bigcup_i S_i \supseteq S}}
    \sum\nolimits_{i=1}^k f(S_i).
\end{equation}
Note the similarity with~\eqref{eq:s-upper-function} -- the difference lies in $S_i$ being disjoint and their union being exactly $S$. The \SAM-lower function is then defined as
$
    \underline{f}_\K[\SAM](S) \coloneqq \max\left[\max_{Y \in \K, Y \subseteq S}f(Y), \max_{X \in \K, S \subseteq X}\left[f(X) - \overline{f}_\K[\SAM](X \setminus S)\right]\right].
$

\begin{restatable}{proposition}{samBounds}
\label{prop:sam-bounds}
$\overline{f}_\K[\SAM]$ and $\underline{f}_\K[\SAM]$ are \SAM-tight upper and \SAM-lower functions.
\end{restatable}

This refinement for \S\ increases computational complexity. To address this, we propose alternative \SAM-upper and \SAM-lower functions that trade off tightness for computational tractability. These approximations can be refined iteratively, starting from the \S-tight upper function. Each iteration alternates between enforcing monotonicity (potentially violating subadditivity) and restoring subadditivity (possibly breaking monotonicity). The process stops when changes fall below a threshold or a maximum number of iterations is reached. The full procedure is detailed in Algorithm~\ref{algo:SAM-bound} (Appendix~\ref{app:sam:algo}). Once the \SAM-upper function is computed, the \SAM-lower function can be computed using $\underline{f}_\K[\SAM]$. We show that
\begin{restatable}{proposition}{samUpperAlgorithm}
\label{prop:sam-upper-algorithm}
Algorithm~\ref{algo:SAM-bound} produces a \SAM-upper function.
\end{restatable}



Once we learn that function $(f,\K)$ lies not only in $\S$ but also in $\SAM$, we can learn more about its values. This idea is reflected in $\S$-divergence and $\SAM$-divergence based on the \S-tight and \SAM-tight upper functions we defined. The following result holds for divergences with $l_1$-norm. 

\begin{restatable}{proposition}{gapSeparation}
\label{prop:gap-separation}
There is a $\SAM$ $(f,\K)$ such that $\gap_f(\SAM,\K) = 0$, while $\gap_f(\S,\K)$ grows at least exponentially in $n$.
\end{restatable}

\subsection{Fractionally subadditive functions}
Fractionally subadditive (or \emph{XOS}) functions are defined as the maximum of a collection of additive set functions. This representation reduces storage from $2^n$ values down to $n \times k$ values, where $k$ is the number of additive functions. Moreover, any subadditive function can be approximated by an XOS function up to a logarithmic factor, further increasing the importance in applications~\cite{feige2006maximizing}.
\begin{definition}
    A set function $f\colon 2^N \to \R$ is XOS function, if there are additive set functions $a_1,\dots,a_s$ such that $f(S) = \max_{i=1,\dots,s}a_i(S)$.
    We denote such functions on ground set of size $n$ by $\XOS$.
\end{definition}

XOS functions are a subclass of subadditive functions, and when based on non-negative additive functions, they are also subadditive and monotone. We present an \XOS-tight upper function that refines the \S-tight bound, and under non-negativity, also improves upon the \SAM-tight bound. This result relies on a key property of XOS functions, explaining their alternative name: \textit{fractionally subadditive}.

\begin{restatable}{proposition}{xosCharacterization}
\label{prop:XOS-characterization}
A set function $f$ is XOS if and only if for every collection $\{\alpha_i,T_i\}_{i=1}^{k}$ with $\alpha_i > 0$ and $T_i \subseteq N$ such that for every $S \subseteq N$, it holds $\sum_{j \in T_i}\alpha_i \geq 1$ for all $j \in S$ and $f(S) \leq \sum_{i=1}^k \alpha_i f(T_i)$. 
\end{restatable}

Following from Proposition~\ref{prop:XOS-characterization}, and based on the constructions of \S-tight upper function, one can derive the bound $\overline{f}_\K[\XOS](S) \coloneqq \min_{\{\alpha_i,T_i\}_{i=1}^k \in M(S)}\sum \alpha_if(T_i)$ for every $S \subseteq N$, where the set $M(S) = \{\{\alpha_i,T_i\}^k_{i=1} \mid T_i \in \K,\ \sum_{i: j \in T_i}\alpha_i \geq 1, \forall j \in S\}$.
\begin{restatable}{proposition}{xosBound}
\label{prop:xos-bound}
Function $\overline{f}_\K[\XOS]$
is \XOS-tight upper function.
\end{restatable}

We do not derive any special lower bound; however, one can, once again, rely on the construction of $\underline{f}_\K[\S]$, substituting the \SAM-tight upper function with the \XOS-tight upper function. 
To further understand the relation between $\SAM$ and $\XOS$, we characterize the set of $\SAM$ functions. To the best of or knowledge, this characterization is not presented anywhere else.
\begin{restatable}{proposition}{samCharacterization}
\label{prop:sam-characterization}
A set function $f \colon 2^N \to \R$ is subadditive monotone if and only if it holds for any $S_1,\dots,S_k$ such that $S \subseteq \bigcup_{i=1}^k S_i$, $f(S) \leq \sum_{i=1}^k f(S_i).$
\end{restatable}

This characterization is equivalent to restricting to sets $S_1,\dots,S_k$ that satisfy $S_i \cap S_j = \emptyset$. This shows that the class differs from subadditive set functions only in the condition $\cup_i S_i \subseteq S$; for subadditive functions, the stronger condition $\cup_i S_i = S$ must hold. 
Moreover, similarly to the relation of \S and \SAM functions, preferring \XOS-tight over \SAM-tight upper function might result in a huge difference in the divergence.
\begin{restatable}{proposition}{xosGapLowerBound}
\label{prop:xos-gap-lower-bound}
There is an $\XOS$ $(f,\K)$ such that $\overline{f}_\K[\SAM] - \overline{f}_\K[\XOS] = 2^k$ for any $k \in \R_+$.
\end{restatable}

\subsection{SCMM functions}
SCMM functions decompose a submodular objective into simpler concave and additive components, retaining the “diminishing returns” property. They were motivated by intractability of general submodular minimization for large combinatorial problems arising in machine learning~\cite{Stobbe2010,Bilmes2017}.
\begin{definition}
    A set function $f\colon 2^N \to \R$ is \emph{SCMM} if $f(S) = \sum_{i}g_i(a_i(S)) + a_{\pm}(S)$ for every $S \subseteq N$,
    where $a_i \colon 2^N \to \R_+$ are nonnegative additive functions, $g_i \colon \left[0,a_i(N)\right] \to \R$ are non-decreasing concave functions with $\phi_i(0)=0$ and $a_{\pm} \colon 2^N \to \R$ is an additive function. We denote such functions on ground set of size $n$ by $\SCMM$.

\end{definition}

SCMM form a strict subset of submodular functions, thus also of XOS functions. In deriving upper and lower completions, we focus on \emph{symmetric submodular functions} and their slight generalization, which we call \emph{concave additive} and which were introduced in~\cite{Ahmed2011}.

\subsubsection{Symmetric submodular functions}
Cardinality-based symmetric submodular functions are studied because they capture the principle of diminishing returns while depending only on subset sizes. Their simple structure facilitates efficient algorithmic solutions in combinatorial optimization. Note that in the literature, the term \emph{symmetric submodular} can also refer to a different property where $f(S) = f(N \setminus S)$~\cite{Queyranne1998}.
\begin{definition}
    A set function $f \colon 2^N \to \mathbb{R}$ is \emph{symmetric submodular} if there is a monotone concave $g \colon [0,|N|] \to \mathbb{R}$ such that $f(S) = g(|S|)$. We denote such functions on ground set of size $n$ by $\SS$.
\end{definition}

The idea of establishing the \SS-tight lower and upper functions is to keep track of the known values of $g$ and to consider linear approximations of the function $g$. From concavity of $g$, we know that no value can lie below this approximation. In a similar manner, the upper bound on subset $S$ is determined a minimum of linear approximations given by pairs of subsets, both either larger, or smaller than $S$. Once again, from concavity, the value of $S$ cannot lie above these approximations. An illustration of this property is shown in \Cref{fig:bounds} in \Cref{app:ss-bounds}. Also, due to their technical nature, the analytical derivations of the bounds are also provided in \Cref{app:ss-bounds}.

\begin{restatable}{proposition}{ssBounds}
\label{prop:ss-bounds}
There exist explicit formulas for \SS-tight upper and lower functions.
\end{restatable}

\subsubsection{Concave additive functions}

We conclude this section by introducing \emph{concave additive} functions, a slight generalization of symmetric submodular functions originally studied in~\cite{Ahmed2011}. These functions naturally arise in applications such as capital budgeting under uncertainty (capturing investor risk aversion), competitive facility location (modeling cannibalization effects), and combinatorial auctions (representing bidders' decreasing marginal valuations). Their special structure facilitates efficient solutions to large-scale discrete optimization problems.

\begin{definition}
A set function $f \colon 2^N \to \mathbb{R}$ is \emph{concave additive} if there exist a concave increasing $g \colon \mathbb{R} \to \mathbb{R}$ and an additive $a \colon 2^N \to \mathbb{R}$ such that $f(S) = g(a(S))$. We denote such functions on a ground set of size $n$ by \CA.
\end{definition}

Earlier works showed that these functions are submodular if and only if the additive function $a$ is either non-negative or non-positive~\cite{Ahmed2011}. Furthermore, they are monotone (and thus belong to our hierarchy) if $g$ is non-negative on $\mathbb{R}_+$. In practical scenarios, the additive component $a$ is often explicitly known or can be bounded based on problem specifics.

Consider an incomplete set function $(f,\K)$ with $f \in \CA$, and let there be known bounds $\underline{a}$ and $\overline{a}$ on the additive function $a$, where without loss of generality $\underline{a} \geq 0$. Lower and upper bounds on $f_\K$ can be constructed via linear interpolation using the known function values from $\K$. The idea behind these bounds closely follows the approach used for symmetric submodular functions, extending the interpolation concept to accommodate the more general additive structure. Explicit formulas for these interpolation-based bounds are somewhat technical and hence provided in the appendix.

\begin{restatable}{proposition}{caBounds}
\label{prop:ca-bounds}
There exist explicit formulas for \CA-tight lower and upper functions.
\end{restatable}

\section{Computing Optimal Subsets to Query}\label{sec:optimism-bias}

With tighter upper and lower completions at hand, we now present three methods to approximately solve the optimal value querying problems. Each method is designed to incorporate these bounds directly; plugging in a specific pair of bounds yields a version tailored to the class those bounds represent. 
Further details are in Appendix~\ref{app: alg specs}.

\subsection{Offline algorithms}
For the offline problem we introduce two methods -- {\sc Offline Optimal} and {\sc Offline Greedy}. Both rely on the same input: a probability distribution $\valueDistribution$ over a given class of set functions we derived bounds for, number of subsets we can query $t$, and number of samples $\kappa$ to estimate the expected divergence. 
The {\sc Offline Optimal} method aims to find a subset of size $t$ from $2^N \setminus \K_0$ that minimizes the expected divergence $\mathbb{E}[\gap_f]$ under $\valueDistribution$. For each such candidate subset, the algorithm approximates the expectation by drawing $\kappa$ samples from $\valueDistribution$, calculating the corresponding divergence $\gap_f$ for each sample, and averaging the results to obtain the mean. The subset with the smallest estimated expected divergence is selected as the solution. This method is formalized as Algorithm~\ref{algo: offline optimal} in Appendix~\ref{app:algo:ofop}. 
The optimality of this approach depends on the quality of the estimation of the divergence, which, to our advantage, improves exponentially with the number of samples.

\begin{restatable}{proposition}{offlineBoundFailProb}
\label{prop:offlineFailProb}
	Let $ \valueDistribution $ be a distribution of bounded set functions with compact support.
    Then the probability that the solution returned by \offlineoptimal{} is not an optimal solution of the offline principal's problem is in $\mathcal{O}_{\valueDistribution, n}( e^{-\kappa}) $.
\end{restatable} 
The {\sc Offline Greedy} method is then a computationally simpler alternative to {\sc Offline Optimal}. At each step, this method selects the next subset $S_t$ that, when added to the existing trajectory $\left\{ S_i \right\}_{i=1}^{t-1}$, minimizes the expected divergence. Again, the expected divergence is approximated by drawing $\kappa$ samples from the distribution $\valueDistribution$, calculating the corresponding divergences, and averaging the results. This approach iteratively builds the solution by greedily selecting the subset that minimizes the expected divergence at each step, hence it cannot perform better than {\sc Offline Optimal}. It is formalized as Algorithm~\ref{algo: offline greedy} in Appendix~\ref{app:algo:ofgr}. We have the following complexity result:
\begin{restatable}{proposition}{timeComplexities}
\label{prop:timecomplex}
Let $T_{\gap}\!(n)$ be the time of computing the divergence.
In our case, $ T_\gap\!(n) $ is $ \mathcal O(2^{2n}) $ in the standard model and $ \mathcal O(n) $ in a parallel model with $\Omega(2^n)$ processors.
Then, in the standard model,
    \offlineoptimal{} $\in\Theta\!\big(\kappa \binom{2^n\!-\!n\!-\!2}{t} T_{\gap}\!(n)\big)$ and 
    \offlinegreedy{} $\in\Theta\!\big(\kappa (2^n\!-\!n) t T_{\gap}\!(n)\big)$.
    In a parallel model with $ \Omega\!\left(2^{3n}\right) $ processors,
    \offlineoptimal{} $\in\mathcal{O}\!\big(\kappa \binom{2^n\!-\!n\!-\!2}{t} T_{\gap}\!(n)\big)$ 
    and 
    \offlinegreedy{} $\in\mathcal{O}\!\left(t T_\gap\!(n)\kappa\right)$.
\end{restatable}

A natural question is under which conditions the local greedy search yields a ``good'' solution. 
It is known that, if the optimized function is supermodular, the locally optimal steps are guaranteed to yield a $(1-1/e)$-approximation of the global optimum~\cite[Proposition 3.4]{Nemhauser1978}.
As long as the size of the ground set is sufficiently small, the divergence for \S and \SS are among these functions:

\begin{restatable}{proposition}{smDivergenceOnFourPlayers}
\label{prop:SM-of-divergence-on-4-players}
Let $n\le4$ and $\C\in\{\S,\SS\}$. If $f\in\C$, then the $\C$‑divergence $\gap_f$ with $l_1$‑norm is supermodular.
\end{restatable}

The primary reason this result holds is that, for  $n \leq 4$, there are relatively few supermodularity conditions to consider, which, in case of \S-divergence are not at all affected by revealing coalitions and in case of \SS-divergence, these conditions are easy to analyze. We believe a similar conclusion can be drawn for the other function classes and their bounding functions discussed in the previous section. However, as a consequence of a known result from cooperative game theory~\cite{uradnik2024reducing}:

\begin{restatable}{proposition}{almostNothing}
\label{prop:almost-nothing}
Let $n\geq 5$ and $f \in \S$. If there exist $i,j,k,l \in N$ such that $g(ij) \le g(jk) \le g(kl) \leq 0$, where $g(S) = f(S) - \sum_{i\in S}f(\{i\})$, and $\left(2^{n-3}-n+2\right)g(kl) < g(ij)$, then the \S-divergence with $l_1$-norm is not supermodular.
\end{restatable}
Note that while similar conclusions to the one in Proposition~\ref{prop:SM-of-divergence-on-4-players} might be drawn for other classes, this is not the case for \Cref{prop:almost-nothing}. The complexity of our approach when $n \geq 5$ prevents us from immediately generalizing the result. \SS-divergence may hence be supermodular even when $n \geq 5$. However, minimizing the \SS-divergence has a straightforward solution: if the querying budget is $n$, querying one subset value for each subset size ensures that the \SS-divergence is zero. 


\subsection{Online algorithm}

Compared to the offline setting, solving the online problem is substantially more challenging. A key difficulty is that the algorithm must compute (or approximate) a restriction of $\valueDistribution$ consistent with previously observed values—hard to do when access to $\valueDistribution$ could be limited to sampling. To address this, we use reinforcement learning, specifically proximal policy optimization (PPO)~\cite{Schulman2017}, an online method that selects actions (subsets) based on past rewards. At each step $\tau \le t$, PPO selects a new subset $S_\tau$ informed by the previous trajectory $\{S_i\}_{i=1}^{\tau-1}$. The reward, which the algorithm aims to maximize, is defined as the negative expected divergence averaged over this trajectory.

\section{Empirical Evaluation}\label{sec: experiments}

We evaluate our algorithms on 3 distributions on different subclasses of subadditive functions. We outline the setup and benchmarks here; exact formulations, implementation details, and extended results are provided in the relevant appendices. 
The (uniform) distributions are:

\begin{itemize}[itemsep=-2pt, topsep=-3pt, leftmargin=2em]
    \item \convex{}: Distribution over monotonically decreasing submodular functions with values in $[-1, 0]$, obtained by negating increasing supermodular functions as in~\cite{uradnik2024reducing, Beliakov2022}.
    \item \xos{}: Distribution over XOS functions formed by taking the pointwise maximum of six random additive functions, each defined by singleton values drawn uniformly from $(0,1)$.
    \item \covg{}: Distribution over functions from combinatorial optimization, particularly set cover problems, defined as $f(S) = \left| \bigcup_{i \in S} X_i \right|$, generated by randomly sampling $X_i \subseteq [2n]$.
\end{itemize}

\textbf{Experimental Setup.} We test the distributions with $n\in\{5,10\}$. For $n=10$, the exhaustive nature of \offlineoptimal{} renders it computationally infeasible, so we only report results for \offlinegreedy{}. \convex{} is not \SAM and therefore uses the \S-divergence, while \xos{} and \covg{} use the tighter \SAM-divergence. We train \ppo{} over 3 million steps, sampling $6 \cdot 2048$ trajectories per iteration and optimizing for 10 epochs each round. Input values $f(S)$ are normalized as described in Appendix~\ref{app:value-normalization}. Offline algorithms use $\kappa=90$ samples to estimate expected divergence. All experiments were conducted on an AMD EPYC 7532 cluster (2.4\,GHz, 18 CPUs) in Python~3. For example, \offlinegreedy{} took 1 min to complete all 25 steps for $n=5$ and 20 h for $n=10$, while solving the same instances took \offlineoptimal{} 43 h ($n=5$). Each figure also contains a simple \textbf{\random{}} strategy, which selects subsets uniformly at random. Further details are in Appendix~\ref{app: alg specs}.

\begin{figure*}[t!]
\centering
\begin{tikzpicture}
\begin{groupplot}[
    group style={group size=3 by 2,
    ylabels at=edge left,
    xlabels at=edge bottom,
    horizontal sep=.8cm,
    vertical sep=.95cm,
    },
	title={\texttt{supermodular}($ n $)},
	title style={yshift=-1.5ex},
	xlabel={Steps},
    every axis x label/.style={
        at={(axis description cs:0.5,-0.1)},
        yshift=-4pt, 
        anchor=north
    },
	ylabel={Divergence},
    yticklabel style={font=\scriptsize},
	scaled y ticks=true,
	scale ticks above exponent=1,
	enlargelimits=false,
	xmin=0,
    height=0.25\textwidth,
    width=0.373\textwidth,
    axis line style={very thin},
    grid style={very thin},
    xmax=24
]
\nextgroupplot[legend to name=sharedlegend, legend columns=-1, title={\convex[5]}]
	\exploitabilityplot{figures/domains/convex_5.txt}{expected_greedy}{\textsc{Offline Greedy}}
	\exploitabilityplot{figures/domains/convex_5.txt}{expected_best_states}{\textsc{Offline Optimal}}
	\exploitabilityplot{figures/domains/convex_5.txt}{eval}{\textsc{PPO}}
	\exploitabilityplot{figures/domains/convex_5.txt}{random_eval}{\textsc{Random}}

    \nextgroupplot[title={\xos[5]}]
	\exploitabilityplot{figures/domains/xos_5.txt}{random_eval}{\random}
	\exploitabilityplot{figures/domains/xos_better_ppo5.txt}{eval_ln}{\ppo}
	\exploitabilityplot{figures/domains/xos_5.txt}{expected_best_states}{\offlineoptimal}
	\exploitabilityplot{figures/domains/xos_5.txt}{expected_greedy}{\offlinegreedy}
    \legend{}
    \nextgroupplot[title={\covg[5]}]
	\exploitabilityplot{figures/domains/covg_fn_generator_5.txt}{random_eval}{\textsc{Random}}
	\exploitabilityplot{figures/domains/covg_fn_generator_5.txt}{eval}{\textsc{PPO}}
	\exploitabilityplot{figures/domains/covg_fn_generator_5.txt}{expected_best_states}{\textsc{Offline Optimal}}
	\exploitabilityplot{figures/domains/covg_fn_generator_5.txt}{expected_greedy}{\textsc{Offline Greedy}}
    \legend{}
    \nextgroupplot[title={\convex[10]}]
	\exploitabilityplot{figures/domains/convex_10.txt}{random_eval}{\textsc{Random}}
	\exploitabilityplot{figures/domains/convex_10.txt}{expected_greedy}{\textsc{Offline Greedy}}
	\exploitabilityplot{figures/domains/convex_10.txt}{eval}{\textsc{PPO}}
    \legend{}
    \nextgroupplot[title={\xos[10]}]
	\exploitabilityplot{figures/domains/xos_10.txt}{random_eval}{\random}
	\exploitabilityplot{figures/domains/xos_10.txt}{eval_ln_lin}{\ppo}
	\exploitabilityplot{figures/domains/xos_10.txt}{expected_greedy}{\offlinegreedy}
    \legend{}
    \nextgroupplot[title={\covg[10]}]
	\exploitabilityplot{figures/domains/covg_fn_generator_10.txt}{random_eval}{\textsc{Random}}
	\exploitabilityplot{figures/domains/covg_fn_generator_10.txt}{expected_greedy}{\textsc{Offline Greedy}}
	\exploitabilityplot{figures/domains/covg_fn_generator_10.txt}{eval}{\textsc{PPO}}
    \legend{}
\end{groupplot}
\node[anchor=north] at (current bounding box.south) {\pgfplotslegendfromname{sharedlegend}};
\end{tikzpicture}%
	\caption{
	    Comparison of divergence across algorithmic steps for various algorithms, showcasing (left) \convex{}, (center) \xos{}, and (right) \covg{} distributions, $n\in\{5,10\}$.
	}
	\label{fig:domain_main}
\end{figure*}
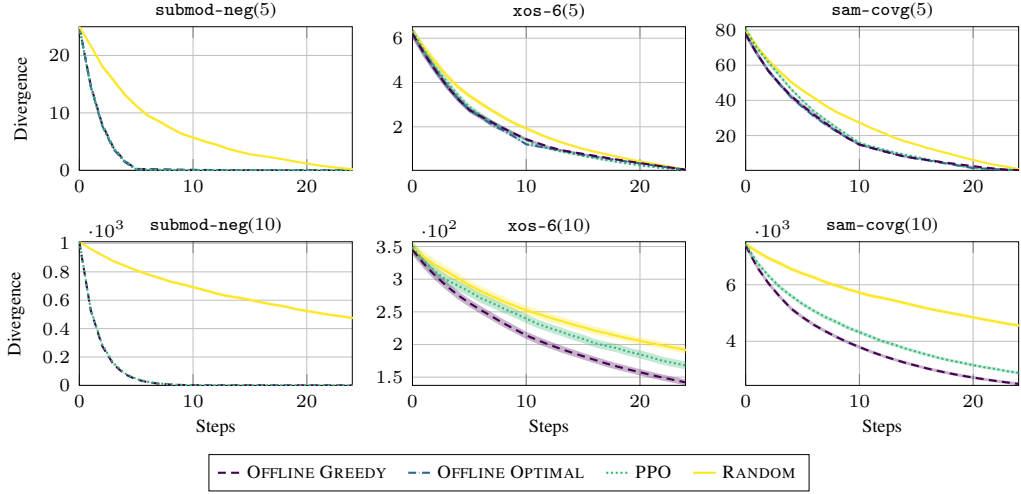

Results are shown in Figure~\ref{fig:domain_main}. A notable observation is that even the \random{} baseline performs reasonably well, indicating that the considered distributions exhibit significant internal structure. This structure allows for accurate approximations from only a few observed values. The \offlinegreedy{}, while still efficient to implement, substantially improves upon \random{}. As $n$ increases, random samples become less informative, amplifying the advantage of sets selected more deliberately by the remaining algorithms. For $n=5$, \offlinegreedy{} performs nearly as well as \offlineoptimal{}, suggesting that it is close to optimal. PPO exploits the information from previous queries to make targeted decisions, achieving slightly lower divergence than \offlinegreedy{} and occasionally outperforming \offlineoptimal{}. For n = $10$, however, PPO fails to generalize effectively and underperforms \offlinegreedy{}. We attribute this to the exponential action space and high-dimensional continuous observations, which pose significant learning challenges.

The relative gains of \offlinegreedy{} and PPO over \random{} are most pronounced on \convex{}. This mirrors findings in cooperative games~\cite{uradnik2024reducing}, where an analogous divergence measure was reduced significantly using only the largest subset values across a wide range of $n$, further supporting that informed selection can approach optimality for structured set functions.

\begin{table}[h!]
    \caption{
    Approximation $\alpha$ minimizing multiplicative error on \covg{} of CDSA \( (f_A, \overline{f}_{\mathcal{K}_A}) \), revealed queried values \( (\underline{f}_{\mathcal{K}_A}, \overline{f}_{\mathcal{K}_A}) \), and \offlinegreedy{} \((\underline{f}_{\mathcal{K}_O}, \overline{f}_{\mathcal{K}_O}) \) with the same budget.
    }
    \label{tab:sampling_comparison}
    \centering
    \begin{tabular}{lccccc}
        \toprule
        Distribution $ \valueDistribution $ & \# Queried & $ \alpha(f_A, \overline f_{\K_A}) $ & $ \alpha(\underline f_{\K_A}, \overline f_{\K_A}) $ & $ \alpha(\underline f_{\K_O}, \overline f_{\K_O}) $ \\
        \midrule
        \covg[5]   & 10.61 & 1.89  & 1.78 & \textbf{1.44} \\
        \covg[8]   & 25.42 & 2.06 & 2.01 & \textbf{1.88} \\
        \covg[10]   & 42.41 & 2.28 & 2.24 & \textbf{2.23} \\
        \bottomrule

    \end{tabular}
\end{table}

Lastly, we test whether our algorithm can construct effective $\alpha$-sketches of \SAM functions. These approximate $f$ from below via $f_A \le f$, minimizing $\alpha(f_A, f) = \max_S \frac{f_A(S)}{f(S)}$, which is equivalent to minimizing the multiplicative error, unlike our original additive error objective. 
We compare against the Cohavi--Dobzinski Sketching Algorithm (CDSA)~\cite{Cohavi2017}, which guarantees $\alpha \in \tilde{\mathcal O}(\sqrt{n})$ for submodular monotone functions. 
For each function \(f \sim \valueDistribution\), CDSA queries a set of subsets \( \mathcal{K}_A \) and builds \(f_A\). 
We then run \offlinegreedy{} with the same budget \(t_A = |\mathcal{K}_A|\) to produce a comparable set \( \mathcal{K}_O \). 
In Table~\ref{tab:sampling_comparison}, we report the ratios \( \alpha(f_A, \overline f_{\mathcal{K}_A}) \), \( \alpha(\underline f_{\mathcal{K}_A}, \overline f_{\mathcal{K}_A}) \), and \( \alpha(\underline f_{\mathcal{K}_O}, \overline f_{\mathcal{K}_O}) \) on \covg{}. 
In all tested scenarios, \offlinegreedy{} achieves the tightest approximation ratio under the same query budget. 
Additional results are in Appendix~\ref{app:k-budget}.

\section{Conclusion}
\label{sec: cenclusion}

In this paper, we study strategies for efficiently mitigating uncertainty within incomplete set functions. We introduce the concept of the \enquote{set function divergence}, which quantifies the size of the set of possible extensions of a partial set function. We show fundamental properties of the set function divergence that enables us to compute it more efficiently for different subclasses of subadditive functions. We focus on reducing the set function divergence through well-informed queries about the unknown values within the incomplete set function, effectively constructing a representation tailored to capture a maximum information about the function in both online and offline fashion. Our findings indicate that our approach significantly outperforms random queries and approaches optimality.

\section*{Acknowledgments}
The authors would like to thank Martin Loebl and Milan Hlad\'{i}k for their insightful comments. 
This work was supported by Charles Univ. project UNCE 24/SCI/008, CoSP Project grant no. 823748 and by the grant P403-22-11117S of the Czech Science Foundation. The authors were supported by the Charles University Grant Agency (GAUK 206523). 
Computational resources were supplied by the project "e-Infrastruktura CZ" (e-INFRA LM2018140) provided within the program Projects of Large Research, Development and Innovations Infrastructures.

\bibliographystyle{ACM-Reference-Format}
\bibliography{main}

\appendix


\newpage

\section{Proofs of Results}

\gapInK*

\begin{proof}
    Let $\hat\K = \K \cup \K_0, \hat\L = \L \cup \K_0$ such that $\K \subseteq \L$. From the definition of the upper and the lower extension, for $T \subseteq N$, it follows
    \begin{equation*}
        \overline{f}_{\hat\K}(T) \ge \overline{f}_{\hat\L}(T) \hspace{0.1in}\text{and}\hspace{0.1in} \underline{f}_{\hat\L}(T) \ge \underline{f}_{\hat\K}(T),
    \end{equation*}
    or equivalently $\overline{f}_{\hat\K} - \underline{f}_{\hat\K} \ge  \overline{f}_{\hat\L} - \underline{f}_{\hat\L}\ge 0 .$
    Now the $\mathbb{C}^n$-divergence is monotonically non-increasing (1.) as long as the norm satisfies $0 \le x \le y \implies \lvert x \rvert \le \lvert y \rvert$. In~\cite{Bauer1961}, it is showed this holds if $\|x\| = \|\lvert x \rvert\|$. Further, from non-negativity and 1., subadditivity follows. 
    Finally, we have
	\begin{align*}
		\gap_{\alpha f + \beta}(\K)
		 & =
		\norm{\overline{\alpha f + \beta}_{\K_0 \cup \K} - \underline{\alpha f + \beta}_{\K_0 \cup \K}} =
		\alpha \norm{\overline{f}_{\K_0 \cup \K} - \underline{f}_{\K_0 \cup \K}}   
        = \alpha\cdot \gap_f(\K).
	\end{align*}
\end{proof}

\sBounds*

\begin{proof}
The first part of the proposition is equivalent to Theorem 1 in~\cite{Masuya2016} with a slight distinction that Theorem 1 deals with superadditivity instead of subadditivity. The second part follows from Theorem 3 in~\cite{Masuya2016}, which states (when translated to subadditivity) that any set function defined for a non-empty set $T\subseteq N$ as
\vspace{-1ex}
\begin{equation}
    f^S(T) \coloneqq \begin{cases}
        \underline{f}_\K[\S](T)  & S\subseteq T,    \\
        \overline{f}_\K[\S](T) & S\not\subseteq T, \\
    \end{cases}
\end{equation}
is an $\S$-extension of $(f,\K)$. For $S \notin \K$, choose $\underline{g}=f^N$ and $\overline{g}=f^S$.
\end{proof}

\samBounds*

\begin{proof}
    First, let us show $g(S) \leq \overline{f}[\SAM](S)$ for every $\SAM$-extension $g$. Denote $S_1^*,\dots,S_k^*$ subsets achieving minimum\footnote{We note these subsets might not be uniquely defined, however, this fact does not violate our proof.} of
    \[
    \min\limits_{\substack{S_1,\dots,S_k \in \K \\
            \bigcup_i S_i \supseteq S}}
    \sum_{i=1}^k f(S_i).
    \]
    We construct $T_1^*,\dots,T_k^*$ satisfying the following conditions:
    \begin{equation*}
        (1)~T^*_i \subseteq S^*_i,\qquad
        (2)~T^*_i \cap T^*_j = \emptyset,\qquad
        (3)~\bigcup T^*_i = S.
    \end{equation*}
    From $(1)$, and monotonicity of $g$, we get $g(T^*_i) \leq g(S^*_i) = f(S^*_i)$ and from $(2)$, $(3)$, we get $g(S) \leq \sum_{i=1}^kg(T^*_i)$, which together imply $g(S) \leq \overline{f}[\SAM](S)$.
    We construct $T^*_1,\dots,T^*_k$ in two steps. First we construct $T_1,\dots,T_k$ as
    \begin{equation}\label{eq:formula}
        T_i \coloneqq S^*_i \setminus \left( S^*_i \cap \bigcup_{j=i+1}^k S^*_j \right)
    \end{equation}
    and then
    \begin{equation}\label{eq:formula2}
        T_i^* = T_i \cap S.
    \end{equation}
    One could view~\eqref{eq:formula} as a process, where we start with $S^*_1$ and construct $T_1$ by eliminating all elements which are present in an intersection with any other set $S^*_i$. For constructing $T_2$, elements from $S^*_1 \cap S^*_2$ remain in the intersection, as these are not present in $T_1$. The same idea is applied to every $T_i$. It is obvious from the construction that $T_i \subseteq S_i^*$ and $T_i \cap T_j = \emptyset$. However, it holds $\bigcup T_i = \bigcup S^*_i$, which might violate $(3)$. This condition is further achieved by construction in~\eqref{eq:formula2}, while preserving the previous two.

    Now, let us prove $\overline{f}[\SAM] \in \SAM(f,\K)$, which concludes it is a \SAM-tight. The monotonicity follows, as for $S \subseteq T \subseteq N$, we have
    \begin{equation*}
        \min\limits_{\substack{S_1,\dots,S_k \in \K \\
            \bigcup_i S_i \supseteq S}}
    \sum_{i=1}^k f(S_i) \leq \min\limits_{\substack{S_1,\dots,S_\ell \in \K \\
            \bigcup_i S_i \supseteq T}}
    \sum_{i=1}^\ell f(S_i)
    \end{equation*}
    and the subadditivity also holds as for any $S,T \subseteq N$, $S \cap T = \emptyset$, we have
    \begin{equation*}
        \min\limits_{\substack{X_1,\dots,X_k \in \K \\
            \bigcup_i X_i \supseteq S \cup T}}
    \sum_{i=1}^k f(X_i) \leq \min\limits_{\substack{S_1,\dots,S_\ell \in \K \\
            \bigcup_i S_i \supseteq S}}
    \sum_{i=1}^\ell f(S_i) + \min\limits_{\substack{T_1,\dots,T_m \in \K \\
            \bigcup_i T_i \supseteq T}}
    \sum_{i=1}^m f(T_i),   
    \end{equation*}
    as any $S_1,\dots,S_\ell,T_1,\dots,T_m$ might be viewed as one of $X_1,\dots,X_k$.

    Next, we proceed with showing $\underline{f}_\K[\SAM](S) \leq g(S)$ for every $\SAM$-extensions $g$. Notice $g(A) \geq g(A \cup B) - g(B)$, which for $A\cup B = X \in \K$ yields
    \begin{equation}
        g(A) \geq f(X) - g(X \setminus A) \geq f(X\setminus A) - \overline{f}_\K[\SAM](X \setminus A).
    \end{equation}
    Therefore, $\underline{f}_\K[\SAM](A) \geq \max_{\substack{X \in \K \\ A \subseteq X}}f(X) - \overline{f}_\K[\SAM](X \setminus A)$. Further, for $Y \in \K$, $Y \subseteq S$, it holds $f(Y) \leq g(S)$, therefore $\underline{f}_\K[\SAM](S) \geq \max_{Y \in \K, Y \subseteq S} f(S)$, which concludes the proof.

\end{proof}

\samUpperAlgorithm*
\begin{proof}
    The initial bound, function $f_0$ is \S-upper function, therefore, it is necessarily also \SAM-upper function. Now assume $f_{k-1}$ is \SAM-upper function. Now if $f^{'}_k(S) < f_{k-1}(S)$ for some $S \subseteq N$, it means there is $X \subseteq S$ satisfying
    \begin{equation*}
        f_{k-1}(S) > f_{k-1}(X) + f_{k-1}(X \setminus S).
    \end{equation*}
    Any \SAM-extension $g$ is bounded by \SAM-tight upper function $\overline{f}_\K$, therefore, from subadditivity, 
    \begin{equation*}
        g(S) \leq g(X) + g(X \setminus S) \leq \overline{f}_\K(X) + \overline{f}_\K(X \setminus S) \leq f_{k-1}(X) + f_{k-1}(X \setminus S). 
    \end{equation*}
    We see that by refining $f_{k-1}(S)$ to $f'_k(S) = \min_{X \subseteq S} f_{k-1}(X) + f_{k-1}(S \setminus X)$, we do not get below any \SAM-extension, thus $f'_k(S)$ is \SAM-upper function. Similarly, if $f_k(S) < f'_k(S)$, there is $T \subseteq N, S \subseteq T$ such that $f_k(S) > f'_k(T)$. Again, since for every \SAM-extension $g$, it holds from monotonicity
    \begin{equation*}
        g(S) \leq g(T) \leq \overline{f}_\K(T) \leq f'(T),
    \end{equation*}
    there is no violation in setting $f_k(S) = \min_{T, S \subseteq T} f^{'}_k(S)$, thus $f_k$ is also \SAM-upper function.
\end{proof}

\gapSeparation*

\begin{proof}
    The result holds for $(f, \mathcal{K})$ where $\mathcal{K} = \mathcal{K}_0$ and $f(N) = f(i) = 1$ for every $i \in N$. The set of $\SAM$-extensions contains only a single set function $g$, for which $g(S) = 1$ for every $S \subseteq N$. Thus, we have
\[
\overline{f}_{\mathcal{K}}[\SAM](S) = \underline{f}_{\mathcal{K}}[\SAM](S) = 1 \quad \text{for every } S \subseteq N
\]
and consequently, \SAM-divergence is zero.

For the set of $\S$-extensions, the \S-tight upper function is less restrictive:  $\overline{f}_{\mathcal{K}}(\S, S) = |S|$ for every $S \notin \K$.
The difference in \S-divergence and \SAM-divergence with $l_1$-norm is at least the difference between the \S-tight upper and \SAM-tight upper functions $\overline{f}_\K[\S] - \overline{f}_\K[\SAM]$, which is \begin{equation*}
        \sum_{s=2}^{n-1} {n \choose s}(s-1) \geq \sum_{s=2}^{n-1} {n \choose s} = 2^n - n - 2.
    \end{equation*}
\end{proof}

\xosCharacterization*

\begin{proof}
    Showing that every XOS function is fractionally subadditive is immediate from the definition. For the opposite implication, one has to use the duality of linear programming. The details can be found in~\cite{feige2006maximizing}.
\end{proof}

\xosBound*

\begin{proof}
    The fact that $g(S) \leq \overline{f}_\K[\XOS](S)$ for every $S \subseteq N$ is immediate and follows similar idea as in the proof of Proposition~\ref{prop:sam-bounds}.

    To show the bound is tight, consider for $S \subseteq N$ any collection $\{\alpha_i,S_i\}$ such that $\sum_{i: j \in S_i}\alpha_i \geq 1$ for every $j \in S$. Now we can express
    \begin{equation}
        \sum_{i=1}^k\alpha_i\overline{f}_\K[\XOS](S_i) = \sum_{i=1}^k\sum_{j=1}^{k_i}\alpha_i\beta_jf(S_{ij}).
    \end{equation}
    For $\gamma_{ij} = \alpha_i\beta_{j}$, it holds for every $k \in S$ that $\sum_{i,j: k \in S_{ij}}\gamma_{i,j} \geq 1$, thus
    \begin{equation}
        \overline{f}_\K(S) =\min_{\substack{\alpha_i,S_i\\ S_i \in \K\\ \sum_{i: j \in S_i} \alpha_i \geq 1\\ \forall j \in S}}\alpha_i f(S_i) \leq \sum_{i=1}^k\sum_{j=1}^{k_i}\gamma_{ij}f(S_{ij}) = \sum_{i=1}^k\alpha_i\overline{f}_\K(S_i).
    \end{equation}
\end{proof}

\samCharacterization*

\begin{proof}
    One can readily see that the conditions of subadditivity and submodularity are among conditions from the statement. Any other condition in the statement can be further derived from \SAM conditions. Consider a partition (with some of the sets possibly empty) of $S$ into $T_1,\dots,T_k$ satisfying $T_i \subseteq S_i$ for every $i$. From subadditivity, $f(S) \leq \sum_{i=1}^kf(T_i)$, which is from monotonicity smaller or equal to $\sum_{i=1}^kf(S_i)$.
\end{proof}

\xosGapLowerBound*

\begin{proof}
    Let $N = \{1,2,3,4\}$, $\K = 2^N \setminus \{ S \subseteq N \mid |S| = 3\}$ and $f \in \XOS$ being defined by additive functions 
    \begin{itemize}
        \item $a^{\{i\}} = 2^k e_i$,
        \item $a^{\{i,j,k\}} = 1 / 2 (e_i + e_j + e_k)$,
        \item $a^N = 2^k / 2 (e_1+e_2+e_3+e_4)$,
    \end{itemize}
    where $k \in \R_+$. It follows $f(\{i\}) = f(\{i,j\}) = 2^k$, $f(\{i,j,k\}) = \frac{3}{2}2^k$ and $f(N) = 2\cdot 2^{k+1}$. now $\overline{f}_\K[\SAM] = 2 \cdot 2^k$, while $\overline{f}_\K[\XOS] = \frac{3}{2} 2^k$.
\end{proof}

\subsection{Completions for \SS-functions and \CA-functions}\label{app:ss-bounds}

For $S \subseteq N$, denote $\overline{S}$ as the smallest set in $\K$ that is larger in size than $S$ and $\underline{S}$ as the largest set in $\K$ that is smaller in size than $S$, formally $\overline{S} = \argmin_{T \in K, |S|\leq |T|}|T|$ and $\underline{S} = \argmax_{T \in \K, |T|\leq |S|}|T|$.
Now the value of subset $S$ for the lower function lies on the red line as in Figure~\ref{fig:SS-lower}. Analytically, it can be expressed as

\begin{equation}\label{eq:ss-lower-function}
\underline{f}_\K[\SS](S) \coloneqq \begin{cases}
    f(T) & \text{if } \exists T \in \K\text{, }|T|=|S|,\\
    f(|\underline{S}|) + \frac{(|S|-|\underline{S}|)(f(\overline{S})-f(\underline{S}))}{|\overline{S}|-|\underline{S}|} & \text{otherwise.}\\
\end{cases}
\end{equation}

In a similar manner, the upper bound on subset $S$ is determined a minimum of linear approximations given by pairs of subsets, both either larger, or smaller than $S$. Once again, from concavity, the value of $S$ cannot lie above these approximations, represented as green lines in Figure~\ref{fig:SS-upper}. Formally, for every $S \subseteq$ denote pairs $S_1$, $S_2$ and $S_3$, $S_4$ as $S_1 = \argmin_{T \in K, |S|< |T|}|T|$ and $S_2 = \argmin_{T \in K, |S_1|< |T|}|T|$ and similarly $S_4 = \argmax_{T \in \K, |T|< |S|}|T|$ and $S_3 = \argmax_{T \in \K, |T|< |S_4|}|T|$.
The \SS-tight upper function is then defined for $S$ as $\overline{f}_\K[\SS](S) = f(T)$ if $\exists T \in \K, |T|=|S|$ and otherwise

\begin{align}
\overline{f}_\K[\SS](S) 
&= \min\bigg\{f(|S_1|) + \frac{(|S| - |S_1|)(f(S_2) - f(S_1))}{|S_2| - |S_1|}, \notag \\
&\qquad\quad f(|S_3|) + \frac{(|S| - |S_3|)(f(S_4) - f(S_3))}{|S_4| - |S_3|} \bigg\} \label{eq:ss-upper-function}
\end{align}

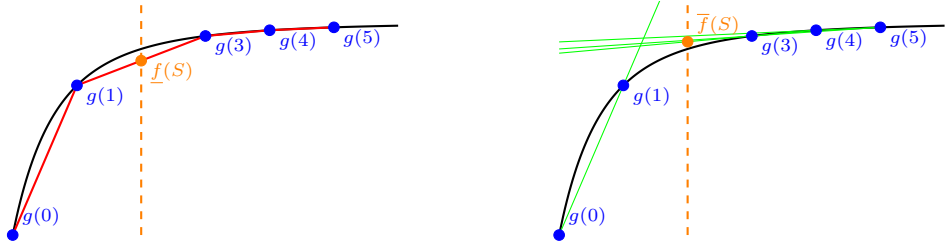
\begin{figure}[h!]
    \centering
    \begin{subfigure}[b]{0.48\textwidth}
        \centering
        \begin{tikzpicture}
            \begin{axis}[height=.7\textwidth,width=\textwidth,enlargelimits=false,ymax=1.1,xtick=\empty,ytick=\empty,axis lines=none]
                \addplot [
                    thick,
                    domain=0:6, 
                    samples=100,
                    smooth,
                ]
                {-1/(x/2+1)^3+1};
                \pgfmathsetmacro\y{-1/(0/2+1)^3+1}
                \addplot[blue, mark=*,nodes near coords={$g(0)$},every node near coord/.append style={anchor=south west}] coordinates {(0, \y)};
                \pgfplotsinvokeforeach{1,3,4,5}{
                    \pgfmathsetmacro\y{-1/(#1/2+1)^3+1}
                    \addplot[blue, mark=*,nodes near coords={$g(#1)$},every node near coord/.append style={anchor=north west,yshift=3}] coordinates {(#1, \y)};
                }
                \pgfmathsetmacro\y{-1/(1+1)^3+1}
                \addplot[orange, dashed, thick] coordinates {(2,0) (2, 1.1)};
                \addplot[red, thick] coordinates {(0, 0) (1, {-1/(1/2+1)^3+1})};
                \addplot[red, thick] coordinates {(1, {-1/(1/2+1)^3+1}) (3, {-1/(3/2+1)^3+1})};
                \addplot[red, thick] coordinates {(3, {-1/(3/2+1)^3+1}) (4, {-1/(4/2+1)^3+1})};
                \addplot[red, thick] coordinates {(4, {-1/(4/2+1)^3+1}) (5, {-1/(5/2+1)^3+1})};
                \addplot[orange, mark=*,nodes near coords={$\underline f(S)$},every node near coord/.append style={anchor=north west,yshift=3}] coordinates {(2, {(-1/(1/2+1)^3+1+-1/(3/2+1)^3+1)/2})};
            \end{axis}
        \end{tikzpicture}
        \caption{The red lines depict the lower completion given by the known values on $g$, i.e., the lower bound on unknown values is given by the red lines.}
        \label{fig:SS-lower}
    \end{subfigure}%
    \hfill
    \begin{subfigure}[b]{0.48\textwidth}
        \centering
        \begin{tikzpicture}
            \begin{axis}[height=.7\textwidth,width=\textwidth,enlargelimits=false,ymax=1.1,xtick=\empty,ytick=\empty,axis lines=none]
                \addplot [
                    thick,
                    domain=0:6, 
                    samples=100,
                    smooth,
                ]
                {-1/(x/2+1)^3+1};
                \pgfmathsetmacro\y{-1/(0/2+1)^3+1}
                \addplot[blue, mark=*,nodes near coords={$g(0)$},every node near coord/.append style={anchor=south west}] coordinates {(0, \y)};
                \pgfplotsinvokeforeach{1,3,4,5}{
                    \pgfmathsetmacro\y{-1/(#1/2+1)^3+1}
                    \addplot[blue, mark=*,nodes near coords={$g(#1)$},every node near coord/.append style={anchor=north west,yshift=3}] coordinates {(#1, \y)};
                }
                \addplot[orange, dashed, thick] coordinates {(2,0) (2, 1.1)};
                \addplot[green] coordinates {(0, 0) (5, {(-1/(1/2+1)^3+1)*5})};
                \pgfmathsetmacro\from{5}
                \pgfmathsetmacro\to{3}
                \addplot[green] coordinates {(\from, {(-1/(\from/2+1)^3+1)}) (0,{(-1/(\from/2+1)^3+1)-\from*((-1/(\from/2+1)^3+1)-(-1/(\to/2+1)^3+1))/(\from-\to)})};
                \pgfmathsetmacro\from{5}
                \pgfmathsetmacro\to{4}
                \addplot[green] coordinates {(\from, {(-1/(\from/2+1)^3+1)}) (0,{(-1/(\from/2+1)^3+1)-\from*((-1/(\from/2+1)^3+1)-(-1/(\to/2+1)^3+1))/(\from-\to)})};
                \pgfmathsetmacro\from{4}
                \pgfmathsetmacro\to{3}
                \addplot[green] coordinates {(\from, {(-1/(\from/2+1)^3+1)}) (0,{(-1/(\from/2+1)^3+1)-\from*((-1/(\from/2+1)^3+1)-(-1/(\to/2+1)^3+1))/(\from-\to)})};
                \addplot[orange, mark=*,nodes near coords={$\overline f(S)$},every node near coord/.append style={anchor=south west,yshift=-1}] coordinates {(2,{(-1/(\from/2+1)^3+1)-(\from-2)*((-1/(\from/2+1)^3+1)-(-1/(\to/2+1)^3+1))/(\from-\to)})};
            \end{axis}
        \end{tikzpicture}
        \caption{Similarly, the green lines depict the upper bounds given the known values on $g$. The upper completion is a minimum over the green lines.}
        \label{fig:SS-upper}
    \end{subfigure}
    \caption{An illustration of lower and upper completions for a function $f(T) = g(\absolute{T})$ with an unknown subset $ S $, such that $\absolute{S} = 1$.}
    \label{fig:bounds}
\end{figure}

\ssBounds*

\begin{proof}
    It follows from the construction described above that $\overline{f}_\mathcal{K}[\SS]$ and $\underline{f}_\mathcal{K}[\SS]$ are \SS-upper and \SS-lower functions, respectively. Any function $f'$ of the form $f'(S) = g(|S|)$ that satisfies either $f'(S) > \overline{f}_\mathcal{K}[\SS](S)$ or $ f'(S) < \underline{f}_\mathcal{K}[\SS](S)$ cannot maintain the concavity of $g$ at the same time.

    The tightness of these bound functions follows from two key observations:
    \begin{enumerate}
    \item The \SS-lower function is actually an \SS-extension. The concave function corresponding to the \SS-lower function is represented by the red lines in Figure~\ref{fig:SS-lower}.
    \item The value of the \SS-upper function for a set $S$ such that  $\nexists T \in \mathcal{K}, \, |T| = |S|$ is attained for a function where $g$ is constructed as a combination of the red lines for the lower completion and the green lines for the upper completion in the segment corresponding to $S$.
    \end{enumerate}
    The formal proof is omitted for the sake of brevity but follows directly from the construction and properties of \SS-upper and \SS-lower functions.
\end{proof}

\caBounds*

\begin{proof}
    The result can be proved in a similar way as Proposition~\ref{prop:ss-bounds}. A \CA-extension $f'$ with a value of coalition $S \notin \K$ outside the range $[\underline{f}_\K[\SS](S),\overline{f}_\K[\SS](S)]$ would be in contradiction with concavity of the corresponding function $g$. The way to prove this is to consider coalitions $\underline{S},\overline{S}$, and $S_1,\dots,S_4$ for which $\underline{f}_\K[\SS](S)$, resp. $\overline{f}_\K[\SS](S)$ is attained and consider the behaviour of $g$, resp. $f'$ on any chain of known coalition extending these.  
\end{proof}




\subsection{Proof of the Sampling Bound}
\label{app:offlineFailProb}

\offlineBoundFailProb*

\begin{proof}
	Let $ \ss^* = \left\{ S^*_i \right\}_{i=1}^t $ be the \emph{true} optimal solution.
	For brevity, denote the set of viable solutions as \[
		\viable = \left\{ \ss \in 2^{2^N \setminus \K_0} \mid \absolute{\ss} = t \right\}.
	\]
	Let $ \hat\gap $ be the \emph{empirical value of the divergence} as seen by \offlineoptimal{}, and $ \gap $ the true value (the true expectation over all $ f \sim \valueDistribution $).

	For our proof, we will need to find the \emph{maximum difference} between divergences of two samples \[
		M = \sup_{\substack{\ss \in \viable \\ f,g \in \supp \valueDistribution}} \absolute{\gap_f(\ss) - \gap_g(\ss)}.
	\]
	observe that this value is bounded, since we assume that $ \valueDistribution $ is bounded over a compact support.

	Now, let $ \ss = \left\{ S_i \right\}_{i=1}^t $ be a non-optimal solution, which means that $ \gap(\ss) > \gap(\ss^*) $.
	We denote the midpoint between their divergences as $ m_\ss = \frac {\gap(\ss) + \gap(\ss^*)}2 $.
	Our aim is now to upper-bound the probability that \offlineoptimal{} sees a lower divergence of $ \ss $ than that of $ \ss^* $.
	If this has happened, that means that either $ \hat\gap(\ss) \leq m $ or $ \hat\gap(\ss^*) \geq m $ (or possibly both).
	We thus want to upper-bound the probability of the event \[
		A(\ss) = \hat\gap(\ss) \leq m \lor \hat\gap(\ss^*) \geq m.
	\]
	By union-bound, we have \[
		\pr{A(\ss)} \leq \pr{\hat\gap(\ss) \leq m} + \pr{\hat\gap(\ss^*) \geq m}.
	\]
	Define $ \delta = m - \gap(\ss^*) $.
	By Hoeffding's inequality \citep{hoeffding1963probability},
	\begin{align*}
		 \pr{\hat\gap(\ss) \leq m}
		 &= \pr{\hat\gap(\ss) - \gap(\ss) \leq -\delta}
		 \leq \pr{\absolute{\kappa \cdot \hat\gap(\ss) - \kappa \cdot \gap(\ss)} \geq \kappa \cdot \delta} \\ &
		 \leq 2 \exp \left( - \frac{2\delta^2 \kappa^2}{\kappa M} \right) 
		 = 2 \exp \left( - \frac{2\delta^2 \kappa}{M} \right).
	\end{align*}
	Note that $ \hat\gap(\ss) $ is an average over $ \kappa $ samples, so $ \kappa \cdot \hat\gap(\ss) $ is a sum.
	Similarly for $ \ss^* $,
	\begin{align*}
		 \pr{\hat\gap(\ss^*) \geq m}
		 &= \pr{\hat\gap(\ss^*) - \gap(\ss^*) \geq \delta}
		 \leq \pr{\absolute{\kappa \cdot \hat\gap(\ss^*) - \kappa \cdot \gap(\ss^*)} \geq \kappa \cdot \delta} \\ &
		 \leq 2 \exp \left( - \frac{2\delta^2 \kappa^2}{\kappa M} \right) 
		 = 2 \exp \left( - \frac{2\delta^2 \kappa}{M} \right).
	\end{align*}
	We get \[
		\pr{A(\ss)} \leq 4 \exp \left( - \frac{2\delta^2 \kappa}{M} \right).
	\]

	To prove the proposition, we must bound the probability that \offlinegreedy{} chooses incorrectly, in other words, the probability that $ A(\ss) $ happens for any $ \ss $.
	To make our job easier, we do not consider only the $ \ss $ of size $ t $, but of all sizes.
	This just gives a looser upper bound, not dependent on $ t $.
	Also, define $ \delta = \min \left\{   \frac{\gap(\ss) - \gap(\ss^*)}2 \mid \ss \in \viable, \gap(\ss) > \gap(\ss^*) \right\} $.
	We reach
	\begin{align*}
		\pr{\left( \exists \ss \in \viable \right)\! \left( A(\ss) \right) }&
		\leq \sum_{ \ss \in \viable} \pr{ \left( A(\ss) \right) }
		\leq \sum_{ \ss \in \viable} 4 \exp \left( - \frac{2\delta_\ss^2 \kappa}{M} \right) \\ &
		\leq 2^{2^n} 4 \exp \left( - \frac{2\delta^2 \kappa}{M} \right) \in \mathcal{O}(e^{-\kappa}),
	\end{align*}
	which is what we wanted to prove.
\end{proof}

\subsection{Time Complexity of \offlineoptimal{} and \offlinegreedy{}}

\timeComplexities*

\begin{proof}[Proof of \offlinegreedy{}, standard model.]
  First, we sample $ \kappa $ functions (we assume sampling of a single game takes $ \mathcal O\!\left(2^n\right) $, i.e., constant time for each subset).
  The algorithm then iterates over $ t $ timesteps. 
  In each timestep, we compute the  $\argmin$, which translates to going over all viable subset structures, of which there are $ 2^n-n-2 $.
  For each subset, we compute the divergence in each sample, reaching $ \kappa \cdot T_\gap(n) $.
  Putting it all together, we get the required time complexity.
\end{proof}

\begin{proof}[Proof of \offlineoptimal{}, standard model.]
  The algorithm iterates over $ \kappa $ samples, for each sample it computes the divergence for every possible subset structure of size $ t $.
  The number of possible subset structures of size $ t $ is $ \binom{2^n-n-2}{t} $.
  The time complexity of computing the divergence is $ T_\gap(n) $.
  Therefore, the time complexity of the \offlineoptimal{} algorithm is $ \Theta\!\left(\kappa \cdot \binom{2^n-n-2}{t} \cdot T_\gap(n)\right) $.
\end{proof}

\begin{proof}[Proof of $ T_\gap $ in the standard model.]
    In the standard model, the main component of computing the divergence is the upper and lower function computation: for each subset $ S \subseteq N $, where $ S \notin \K $, we have to consider all its subsets.
    Hence, we reach the $ \mathcal O\!\left(2^{2n}\right) $ time complexity.
\end{proof}

\begin{proof}[Proof of $ T_\gap $ in a parallel model.]
    In the parallel model, we can perform all the computations for all subsets simultaneously, however, we then need to somehow compute the minimum for each, and then sum them up.
    Both can be done logarithmically, which in our case means $ \mathcal O\!\left(\log 2^n\right) = \mathcal O\!\left(n\right) $.
\end{proof}

\begin{proof}[Proof of parallel model.]
    First observe that the bound we gave for \offlineoptimal{} is the same as in the standard model.
    The proof given above works as an upper bound on time complexity in a parallel model, while allowing us to use the improved $ T_\gap $ bound.

    For \offlinegreedy{}, assuming we have $ \Omega\!\left(2^{3n}\right) $ processors, we can even compute the divergence for all the $ \mathcal O\!\left(2^n\right) $ subsets $ S \subseteq N $ we wish to add to $ \K_t $ at once, which allows us to improve the time complexity by a factor of $ \left( 2^n-n \right) $.
\end{proof}

\subsection{Concavity of Divergence}\label{app:concavity}

Before we prove Proposition~\ref{prop:SM-of-divergence-on-4-players}, we need to prepare several lemmas. In the rest of this appendix, we abuse the notation in the following sense. We write $i$ instead of $\{i\}$, $ij$ instead of $\{i,j\}$ and $S$ instead of $\{S\}$. As we prove Proposition~\ref{prop:SM-of-divergence-on-4-players} for \S and \SS separately, write in the rest of the appendix simply $\overline{f}_\K$, $\underline{f}_\K$ when $\S$ and \SS are obvious from the context.

\begin{lemma}
    \label{lem:bound-funcs-monotone}
    Let $\K_0 \subseteq \K\subseteq 2^N \setminus S$ and $f\in \S$. Then $\forall T\in 2^N$
    \begin{equation*}
        \underline{f}_\K[\S](T) \le \underline{f}_{\K \cup S}[\S](T)
        \hspace{2ex}
        \text{and}
        \hspace{2ex}
        \overline{f}_\K[\S](T) \ge \overline{f}_{\K \cup S}[\S](T).
    \end{equation*}    
\end{lemma}
\begin{proof}
    By definition of the lower function, a larger set of known subsets can only increase the value of the lower function.
    Consequently, the upper function cannot increase if the set of known values becomes larger.
\end{proof}
\begin{lemma}\label{lem:change-of-bounds}
    Let $\K_0 \subseteq\K \subseteq 2^N$, $S\in 2^N\setminus\K$, and $T\in 2^N$.
    Then
    \begin{equation*}
        T\subsetneq S \implies \overline{f}_\K[\S](T) = \overline{f}_{\K\cup S}[\S](T),
    \end{equation*}
    and
    
    \begin{equation*}
        S\subsetneq T \implies \underline{f}_\K[\S](T) = \underline{f}_{\K\cup S}[\S](T).
    \end{equation*}
    
\end{lemma}
\begin{proof}
    As $T \subsetneq S$, revealing $f(S)$ does not appear in any partition of $T$, thus $\overline{f}_\K(T) = \overline{f}_{\K\cup S}(T)$.

    The value of $T$ for $\underline{f}_\K$, resp. $\underline{f}_{\K \cup S}$ is given by considering maximum over $X \in \K$, $T \subseteq X$. As $S \subsetneq T$, it does not appear as a possible $X$. The only possibility for the change of the upper value would be that by revealing $S$, we get $\overline{f}_{\K\cup S}(X \setminus T) \neq \overline{f}_{\K}(X \setminus T)$ for some $X$. This means $S \subseteq X \setminus T$, but this is not possible as $S \subsetneq T$.
\end{proof}

\begin{lemma}
\label{lem:unaffected-bounds}
    Let $\K_0 \subseteq\K \subseteq 2^N\setminus S$ and $T\in 2^N$, such that $\emptyset \neq S\cap T$, $S\not\subseteq T$ and $T \not\subseteq S$. Then
    \begin{equation*}
        \overline{f}_\K[\S](T) = \overline{f}_{\K\cup S}[\S](T),
        \hspace{3ex}
        \text{and}
        \hspace{3ex}
        \underline{f}_\K[\S](T) = \underline{f}_{\K\cup S}[\S](T).
    \end{equation*}    
\end{lemma}
\begin{proof}
    Follows immediately from the definition of the upper/lower functions.
\end{proof}

The divergence $\Delta_f$ of a subadditive function $f$ is supermodular, if for every $\hat{\K} \subseteq 2^N \setminus \left( \K_0\cup S\cup Z\right)$, it holds
\begin{equation}
\label{app: eq: supermodularity constraints}
    \gap_f(\hat{\K}\cup S\cup Z) - \gap_f(\hat{\K}\cup S) \ge \gap_f(\hat{\K}\cup Z) - \gap_f(\hat{\K}).
\end{equation}
Denote $\K = \hat{\K} \cup \K_0$. By definition of the divergence and the fact that $\lvert \overline{f}_\K(T) - \underline{f}_\K(T) \rvert = \overline{f}_\K(T) - \underline{f}_\K(T)$, the submodularity condition can be rewritten as
\begin{multline*}
     \sum_{T\in 2^N}\left(\overline{f}_{\L\cup Z}(T) - \underline{f}_{\L\cup Z}(T) 
     - \overline{f}_{\L}(T) +\underline{f}_{\L}(T)\right)
     \ge\\
     \sum_{T\in 2^N}\left(\overline{f}_{\K\cup Z}(T) - \underline{f}_{\K\cup Z}(T) - \overline{f}_{\K}(T) + \underline{f}_{\K}(T)\right).
\end{multline*}
where $\L = \K\cup S$.

To prove Proposition~\ref{prop: gap in K}, we show an even stronger condition holds, i.e., $\forall T \in 2^N$
\begin{multline}
    \label{eq:SM-stronger-constraint}
    \overline{f}_{\K\cup S\cup Z}(T) - \underline{f}_{\K\cup S\cup Z}(T) 
     - \overline{f}_{\K\cup S}(T) +\underline{f}_{\K\cup S}(T)
     \ge\\
     \overline{f}_{\K\cup Z}(T) - \underline{f}_{\K\cup Z}(T) - \overline{f}_{\K}(T) + \underline{f}_{\K}(T).
\end{multline}

For $T \in \K \cup S$, condition~\eqref{eq:SM-stronger-constraint} holds for any superadditive $f\colon 2^N \to \mathbb{R}$ with $N$ of arbitrary size. To see this, notice $\overline{f}_{\K \cup S \cup Z}(T) = \underline{f}_{\K \cup S \cup Z}(T)$ and $\overline{f}_{\K \cup S}(T) = \underline{f}_{\K \cup S}(T)$, thus~\eqref{eq:SM-stronger-constraint} reduces to
\[
\overline{f}_{\K\cup S}(T) - \underline{f}_{\K\cup S}(T) \le \overline{f}_{\K}(T) - \underline{f}_{\K}(T).
\]
This inequality follows from Lemma~\ref{lem:bound-funcs-monotone}. For the sake of clarity, we split the rest of the proof into two lemmas based on the size of $N$.
\smDivergenceOnFourPlayers*
\begin{proof}[Proof for $n=3$ and $\S$]
    Since, only subsets of size 2 are possibly unknown, their lower and upper bounds do not change when another subset of size 2 is revealed. Thus, for $\K_0 \subseteq \K \subseteq 2^N \setminus \{S,Z\}$ and $T \notin \K\cup S$, it holds 
    \[\overline{f}_{\K \cup S \cup Z}(T) - \underline{f}_{\K\cup S\cup Z}(T) = \overline{f}_{\K \cup S}(T) - \underline{f}_{\K \cup S}(T)\]
    and similarly
    \[\overline{f}_{\K \cup Z}(T) - \underline{f}_{\K\cup Z}(T) = \overline{f}_{\K}(T) - \underline{f}_{\K}(T),\]
    which means~\eqref{eq:SM-stronger-constraint} holds.
\end{proof}

\smDivergenceOnFourPlayers*
\begin{proof}[Proof for $n=4$ and $\S$]
To prove~\eqref{eq:SM-stronger-constraint} for $\K_0 \subseteq \K \subseteq 2^N \setminus \{S,Z\}$ and $T \notin \K \cup S$, we distinguish several cases based on the relation between $Z$ and $T$. All of the cases follow a similar pattern.
\begin{enumerate}
    \item $Z \subsetneq T$:
    
    Since $\lvert N \rvert = 4$ and $Z,T$ are unknown in $\K \cup S$, they must be of form $Z = \{i,j\}$ and $T = \{i,j,k\}$. By Lemma~\ref{lem:change-of-bounds}, $\underline{f}_{\K \cup Z}(T) = \underline{f}_\K(T)$ and $\underline{f}_{\K \cup S \cup Z}(T) = \underline{f}_{\K \cup S}(T)$, thus~\eqref{eq:SM-stronger-constraint} reduces to
    \[
    \overline{f}_{\K \cup S \cup Z}(T) - \overline{f}_{\K \cup S}(T) \ge \overline{f}_{\K \cup Z}(T) - \overline{f}_{\K}(T).
    \]
    For a contradiction, suppose the converse holds. As, by Lemma~\ref{lem:bound-funcs-monotone}, $\overline{f}_{\K \cup S \cup Z}(T) - \overline{f}_{\K \cup Z}(T) \ge 0$, this means $\overline{f}_{\K \cup Z}(T) > \overline{f}_{\K}(T)$, which leads, by Lemma~\ref{lem:bound-funcs-monotone}, to a contradiction.
    
    \item $T \subsetneq Z$:
    Similarly to the first case, we have $T = \{i,j\}$, $Z = \{i,j,k\}$ and by Lemma~\ref{lem:change-of-bounds},~\eqref{eq:SM-stronger-constraint} reduces to
    \[
    \underline{f}_{\K \cup S}(T) - \underline{f}_{\K\cup S\cup Z}(T) \ge \underline{f}_\K(T)-\underline{f}_{\K \cup Z}(T).
    \]
    For a contradiction, suppose the converse holds. It means $\underline{f}_\K(T) > f_{\K \cup Z}(T)$, which by Lemma~\ref{lem:bound-funcs-monotone} leads to a contradiction.
    \item $Z = T$:
    It holds $\overline{f}_{\K \cup S \cup Z}(T) = \underline{f}_{\K \cup S\cup Z}(T)$ and $\overline{f}_{\K \cup Z}(T) = \underline{f}_{\K \cup Z}(T)$, thus~\eqref{eq:SM-stronger-constraint} reduces to
    \[
    \overline{f}_\K(T) - \underline{f}_\K(T) \ge \overline{f}_{\K \cup S}(T) - \underline{f}_{\K \cup S}(T)
    \]
    which holds by Lemma~\ref{lem:bound-funcs-monotone}.
    In this case, $Z = \{i,j\}$, $T = \{k,\ell\}$ and~\eqref{eq:SM-stronger-constraint} reduces to
    \[
    \underline{f}_{\K \cup S}(T) - \underline{f}_{\K \cup S \cup Z}(T)\ge \underline{f}_{\K}(T) - \underline{f}_{\K \cup Z}(T)
    \]
    as $\overline{f}_{\K \cup S\cup Z}(T) = \overline{f}_{\K \cup S}(T)$ and $\overline{f}_{\K \cup Z}(T) = \overline{f}_{\K}(T)$. For a contradiction, if $<$ holds, similarly to previous cases, $\underline{f}_{\K}(T) > \underline{f}_{\K \cup S}(T)$, a contradiction with Lemma~\ref{lem:bound-funcs-monotone}.

    \item $Z \cap T \neq \emptyset$ and $Z \not\subseteq T$ and $T \not\subseteq Z$: By Lemma~\ref{lem:unaffected-bounds},
    condition~\eqref{eq:SM-stronger-constraint} reduces to $0 \ge 0$.
\end{enumerate}
    
\end{proof}

The proof of Proposition~\ref{prop:SM-of-divergence-on-4-players} for \SS functions follows the same argument as in the case of \S functions, except that some steps simplify.

We conclude the appendix with the proof of Corollary~\ref{prop:almost-nothing}. We note this proof is a modification of Proposition 6 in~\cite{uradnik2024reducing}.

\almostNothing*

\begin{proof}
	We use the easily derived result, which states that the \S-divergence  of $(N,w)$ is supermodular if and only if the \S-divergence of $(N,v)$ is supermodular. We focus on the divergence of $(N,v)$ as this makes our analysis simple and at the same time, the condition easy to check.

	The supermodularity condition for the divergence is
	\begin{equation}
		\label{app: eq: supermodularity constraints}
		\gap_{(N,v)}(\hat{\K}\cup S\cup Z) - \gap_{(N,v)}(\hat{\K}\cup S) \ge \gap_{(N,v)}(\hat{\K}\cup Z) - \gap_{(N,v)}(\hat{\K}),
	\end{equation}
	where $\hat{\K} \subseteq 2^N \setminus \K_0$, which can be rewritten as


	\begin{multline}
		\label{app: eq: supermodularity constraints2}
		\sum_{T \subseteq N}\left(\overline{v}_{\hat{\K}\cup S\cup Z}(T) - \underline{v}_{\hat{\K}\cup S\cup Z}(T) - \overline{v}_{\hat{\K}\cup S}(T) + \underline{v}_{\hat{\K}\cup S}(T)\right) \ge \\
		\sum_{T \subseteq N}\left(\overline{v}_{\hat{\K}\cup Z}(T) - \underline{v}_{\hat{\K}\cup Z}(T) - \overline{v}_{\hat{\K}}(T) + \underline{v}_{\hat{\K}}(T)\right).
	\end{multline}

	We obtain the result by setting $\K = \K_0 \cup \{\{j,k\}\}$, $S=\{i,j\}$, and $Z=\{k,l\}$.
	Denote $v(ij) = v(jk) - \varepsilon$ and $v(kl) = v(jk) + \delta$ for some $\varepsilon, \delta \ge 0$.
	For these values, we investigate each of the terms separately.
	Let
	\begin{align}
		L_T & = \overline{v}_{\K\cup S\cup Z}(T) - \underline{v}_{\K\cup S\cup Z}(T) - \overline{v}_{\K\cup S}(T) + \underline{v}_{\K\cup S}(T), \\
		R_T & = \overline{v}_{\K\cup Z}(T) - \underline{v}_{\K\cup Z}(T) - \overline{v}_{\K}(T) + \underline{v}_{\K}(T).
	\end{align}
	In the reminder of the proof, we distinguish different cases based on the relation of $T$ and $\{i,j,k,l\}$ for all $T\in 2^N\setminus \K_0$.
	\begin{enumerate}
		\item $\{i,j,k,l\} \subseteq T\subsetneq N$:

		      Using Lemma~\ref{lem:change-of-bounds}, we get
		      \begin{align*}
			      L_T & = \overline{v}_{\K \cup S}(T) - \overline{v}_{\K\cup S\cup Z}(T), \\
			      R_T & = \overline{v}_{\K}(T) - \overline{v}_{\K\cup Z}(T),
		      \end{align*}
		      where
		      \begin{align*}
			      L_T & = v(jk) - \varepsilon - v(jk) + \varepsilon - v(jk) - \delta  \\
			      R_T & = v(jk) - v(jk)  = 0.
		      \end{align*}
		      It holds $L_T = R_T - v(jk) - \delta = R_T - v(kl)$.
		\item $\{i,j,k,l\}\subseteq N \setminus T$:

		      In this case, the upper bound on $v(T)$ is unaffected by the knowledge of $S,Z$, so
		      \begin{align*}
			      L_T & =  \underline{v}_{\K \cup S\cup Z}(T)
			      - \underline{v}_{\K \cup S}(T),                              \\
			      R_T & =  \underline{v}_{\K\cup Z}(T) - \underline{v}_{\K}(T).
		      \end{align*}
		      Since the only superset of $T$ in $\K\cup S \cup Z$ is $N$, the lower bound reduces to
		      \begin{equation*}
			      \underline{v}_\mathcal{L}(T) =
			      v(N) - \overline{v}_\mathcal{L}(N\setminus T),
		      \end{equation*}
		      for every $\mathcal{L}\in \{\K,\K\cup S, \K \cup Z, \K \cup S\cup Z\}$. Thus
		      \begin{align*}
			      L_T & =  \overline{v}_{\K \cup S}(N\setminus T)- \overline{v}_{\K \cup S \cup Z}(N\setminus T), \\
			      R_T & =  \overline{v}_{\K}(N\setminus T)- \overline{v}_{\K \cup Z}(N\setminus T).
		      \end{align*}
		      Further, $\{i,j,k,l\} \subseteq N\setminus T$, which reduces this case to the previous one, thus $L_T = R_T - v(jk) - \delta = R_T - v(kl)$.

		\item $S=T$:

		      In this case, since $\underline{v}_{\mathcal{L}}(T) = \overline{v}_\mathcal{L}(T) = v(T)$ if $T \in \mathcal{L}$ for every $\mathcal{L}\in \{\K,\K\cup S, \K \cup Z, \K \cup S\cup Z\}$ we get
		      \begin{align*}
			      L_T & =  0,                                                                                                                               \\
			      R_T & =  \underline{v}_{\K\cup Z}(T)- \underline{v}_{\K}(T) = \overline{v}_{\K}(N \setminus T) - \overline{v}_{\K \cup Z}(N \setminus T) = 0.
		      \end{align*}
		      Thus $L_T = R_T$ holds.

		\item $Z=T$:

		      In this case, by the fact that $\underline{v}_{\mathcal{L}}(T) = \overline{v}_\mathcal{L}(T) = v(T)$ if $T \in \mathcal{L}$, it holds
		      \begin{align*}
			      L_T & =  \underline{v}_{\K\cup S}(T)- \overline{v}_{\K\cup S}(T), \\
			      R_T & = \underline{v}_{\K}(T)- \overline{v}_{\K}(T).
		      \end{align*}
		      Since $S \not\subseteq T$, it holds $\overline{v}_{\K \cup S}(T) = \overline{v}_\K(T)$, we can rewrite
		      \begin{align*}
			      L_T & = \underline{v}_{\K \cup S}(T) = v(N) - \overline{v}_{\K\cup S}(N \setminus T)= v(N) - v(ij), \\
			      R_T & = \underline{v}_{\K}(T) = v(N) - \overline{v}_{\K}(N \setminus T) = v(N).
		      \end{align*}
		    Since $v(ij) \leq 0$, $L_T \geq L_T + v(ij) = R_T$.

		\item $\emptyset \neq Z\cap T, Z \not\subseteq T$, and $T \not\subseteq Z$ or $\emptyset\neq S\cap T, S \not\subseteq T$, and $T \not\subseteq S$:

		      By Lemma~\ref{lem:unaffected-bounds}, one immediately arrives at $L_T=R_T=0$.
	\end{enumerate}

	A necessary condition for supermodularity of the \S-divergence is 
    \begin{equation}\label{eq:sum}
        \sum_{T\subseteq N}(L_T - R_T) \ge 0.
    \end{equation}

	If the sum is negative, the \S-divergence is thus not supermodular. In the sum, the only coalitions that matter are those from cases one, two, and four. Specifically, in terms one and two, $L_T - R_T \leq 0$, thus case four has to compensate for these in order for supermodularity to be satisfied.
	Specifically, sum in~\eqref{eq:sum} is equal to
	\[
    \sum_{T\subseteq N}(L_T - R_T) = \left(c_1(n) + c_2(n)\right) v(kl) - v(ij),
    \]
	where $c_1(n) = 2^{n-4}-1$ and $c_2(n)=2^{n-4}-n+3$ is the number of terms satisfying conditions 1 and 2, respectively. Thus, supermodularity is violated, if 
	\[
    \left(2^{n-3}-n + 2\right) v(kl) < v(ij).
	\]
	The coefficient for $v(kl)$ is equal to $1, 4,11,26$, $\dots$ for $n = 5, 6,7,8$, $\dots$ and the sequence can be shown to be increasing.
\end{proof}

\section{Computing bounds for \SAM functions}
\label{app:sam:algo}

\begin{algorithm}[h!]
	\caption{\sc Numerical approach to computing \SAM-upper function}
	\label{algo:SAM-bound}
	\SetAlgoLined
	\DontPrintSemicolon
	\KwIn{incomplete set function $(f,\K) \in \SAM$, number of steps $t$, iteration change limit $\varepsilon$}
	\vspace{1ex}
	\ForEach{$S \subseteq N$}{
		$f_0(S) \gets \min_{S_1,\dots,S_k \in \K, S_i \cap S_j = \emptyset, \bigcup_i S_i = S}\sum_{i=1}^k f(S_i)$\;
		$f_0(S) \gets \min_{S \subseteq T} f_0(T)$\;
	}
    
    \For{$k \gets 1$ \KwTo $t$}{
        \For{$s \gets n$ \KwTo $1$}{
        	\ForEach{$S \subseteq N, |S| = s$}{
    	        $f_k'(S) \gets \min_{X \subseteq S} f_{k-1}(X) + f_{k-1}(X \setminus S)$\;
    		$f_k(S) \gets \min_{S \subseteq T} f'_k(T)$\;
            }
        }
        \If{$\max_{S \subseteq N}|f_{k}(S) - f_{k-1}(S)| < \varepsilon$}{
    	            \KwRet{$f_k$} \;
    	        }
    } 
    \KwRet{$f_t$}\;
\end{algorithm}



\section{Pseudocode for \offlineoptimal{} }
\label{app:algo:ofop}

\begin{algorithm}[h!]
	\caption{\sc Offline Optimal}
	\label{algo: offline optimal}
	\SetAlgoLined
	\DontPrintSemicolon
	\KwIn{distribution of subadditive functions $\valueDistribution$, number of steps $t$, number of samples $\kappa$}
	\vspace{1ex}
	$ \overline \K \gets 2^N \setminus \K_0 $\;
	$G \gets \left\{ \right\}$ \tcp*{trajectories \& their $\mathbb{E}[\gap]$}
	\For(\tcp*[f]{each trajectory}){$\mathcal{S} \subseteq \overline \K: \absolute{\mathcal{S}} = t$}
	{
		$\mu \gets 0$\;
		\For(\tcp*[f]{approx. $\mathbb{E}[\gap]$}){$j \in \{1, \dots, \kappa\}$}
		{
			$ f \sim \valueDistribution$\;
			$\mu \gets \mu + \gap_f(\mathcal{S})$\;
		}
		$ \mu \gets \mu \slash \kappa $\;
		$ G \left[ \mathcal{S} \right] \gets \mu $
	}
	$ \left\{ S_i \right\}_{i=1}^t \gets \argmin_{\mathcal{S} \subseteq \overline{\K}: \absolute{\mathcal S} = t} G \left[ \mathcal{S} \right]$\;
	\Return{$\{S_i\}_{i=1}^{t}$}\;
\end{algorithm}

\section{Pseudocode for \offlinegreedy{}}
\label{app:algo:ofgr}

\begin{algorithm}[h!]
	\caption{\sc Offline Greedy}
	\label{algo: offline greedy}
	\SetAlgoLined
	\DontPrintSemicolon
	\KwIn{distribution of subadditive functions $\valueDistribution$, number of steps $t$, number of samples $\kappa$}
	\vspace{1ex}
	\SetKwFunction{FNC}{\sc Offline Greedy}
	\SetKwProg{Fn}{Function}{:}{}
        \If{$t>1$}{
	$\{S_i\}_{i=1}^{t - 1} \gets $
        {\FNC{$\valueDistribution, t-1$}}\;
        }
	$ \overline \K \gets 2^N \setminus ( \K_0 \cup \left\{ S_i \right\}_{i=1}^{t-1} ) $\;
	$G \gets \left\{ \right\}$ \tcp*{trajectories \& their $\mathbb{E}[\gap]$}
	\For(\tcp*[f]{each trajectory}){$S \in \overline \K$}
	{
		$\mu \gets 0$\;
		\For(\tcp*[f]{approx. $\mathbb{E}[\gap]$}){$j \in \{1, \dots, \kappa\}$}
		{
			$ f \sim \valueDistribution$\;
			$\mu \gets \mu + \gap_f(\left\{ S_i \right\}_{i=1}^{t-1} \cup \left\{ S \right\})$\;
		}
		$ \mu \gets \mu \slash \kappa $\;
		$ G \left[ S \right] \gets \mu $
	}
	$S_{t} \gets \argmin_{S\in \overline \K} G \left[ S \right]$\;
	\Return{$\{S_i\}_{i=1}^{t}$}\;
\end{algorithm}



    



\section{Algorithm Specifications}
\label{app: alg specs}

\paragraph{Oracle and Offline Algorithms}
Greedy strategies are computationally straightforward compared to their optimal counterparts.
A single step using a greedy approach demands $ \mathcal{O}(g \cdot 2^n) $ time for a single sample.
Here, $ g $ represents the time required to compute the divergence.

On the other hand, optimal algorithms prove to be computationally intensive due to the necessity of examining every sequence of actions, that is, every subset of $ 2^N $.
For each sample, the time complexity amounts to $ \mathcal{O}(g \cdot 2^{2^n}) $.
Despite our efforts to parallelize the computation where possible, given our available resources, we were only able to compute the complete optimal strategies for set function on ground sets of size up to 5.
The computation for ground set of size 5 took roughly 300 hours.
We estimate, that to compute the optimal strategies for ground set of size 6, for all steps, would take over 100 years.

The online variants are easier to parallelize, since each sample can be computed and evaluated independently, in contrast to the offline variants, where all the samples need to be computed, put together, averaged, and then finally evaluated.

When estimating the expectation with respect to $\valueDistribution$, we used $ \kappa = 90 $ samples.

\paragraph{Reinforcement Learning}
We apply reinforcement leaning~\cite{sutton2014reinforcement} to approximate the optimal strategy of the online principal's problem.
Namely, we use the Proximal policy optimization~({\sc PPO})~\cite{Schulman2017}.
We want to find a strategy of the principal which efficiently minimizes the average divergence.
As such, we train {\sc PPO} to minimize the divergence
\begin{equation*}
	r(\K_{\tau-1}, S_\tau) = -\gap_f(\K_{\tau-1}\cup \{S_\tau\}),
\end{equation*}
at every step, which provides a stronger learning signal compared to the final reward.
This is equivalent to training over a distribution of the online principal's problems with uniformly distributed size $t$.

In our implementation, we parametrize both actor and critic of the {\sc PPO} algorithm with a two-layer fully-connected neural network with 64 hidden units and ReLU activation each.
To optimize the surrogate {\sc PPO} objective, we used the Adam optimizer.
The rest of the hyperparameteres can be found in Table~\ref{tab: hyperparameters}.

\begin{table}[t]
	\caption{Hyperparameters used during training.}
	\label{tab: hyperparameters}
\centering
	\begin{tabular}{c|c|c}
		\hline
		Parameter  & Value             & Description                    \\
		\hline
		\hline
		$\alpha $ & $3\cdot10^{-4}$   & Learning rate                  \\
		$\beta$    & $0$               & Entropy regularization         \\
		$\gamma$   & 0.99              & Reward discounting rate        \\
		$\lambda$  & 0.95              & Generalized advantage estimate \\
		$\varepsilon$ & 0.2            & Surrogate clip range           \\
		$B$        & $5\cdot 10^{4}$   & Rollout buffer size            \\
		$M$        & 0.5               & Max gradient norm              \\
		$n_e$      & 10                & Number of training epochs      \\
		\hline\end{tabular}
\end{table}

\paragraph{Random Algorithm}
The {\sc Random} algorithm is computationally simple, requiring only $\mathcal{O}(g)$ time to compute a single sample.
As such, it took only a few minutes to compute for ground sets of size 5.
We again used $ \kappa=90 $ samples to approximate the expectation and the standard error of the mean.

\section{Oracle algorithms}\label{app:oracle-algs}
Following are the pseudocodes of {\sc Oracle} algorithms. 
\begin{algorithm}[h!]
	\caption{\sc Oracle Optimal}
	\label{algo: online optimal}
	\SetAlgoLined
	\DontPrintSemicolon
	\KwIn{characteristic function $f\in \S$, number of steps $t$}
	\vspace{1ex}
	$ \overline \K \gets 2^N \setminus \K_0 $\;
	$ \left\{ S_i \right\}_{i=1}^t \gets \argmin_{\mathcal S \subseteq \overline \K: \absolute{\mathcal S} = t} \gap_f(\K_0\cup \mathcal{S})$\;
	\Return{$\{S_i\}_{i=1}^{t}$}\;
\end{algorithm}

\begin{algorithm}[h!]
	\caption{\sc Oracle Greedy}
	\label{algo: online greedy}
	\SetAlgoLined
	\DontPrintSemicolon
	\KwIn{characteristic function $f\in \S$, number of steps $t$}
	\vspace{1ex}
	\SetKwFunction{FNC}{\sc Oracle Greedy}
	\SetKwProg{Fn}{Function}{:}{}
        \If{$t>1$}{
	$\{S_i\}_{i=1}^{t - 1} \gets $ {\FNC{$f, t-1$}}\;
        }
	$ \overline \K \gets 2^N \setminus (\K_0 \cup \left\{ S_i \right\}_{i=1}^{t - 1} $)\;
	$S_{t} \gets \argmin_{S\in \overline \K} \gap_f(\K_0\cup \{S_i\}_{i=1}^{t - 1}\cup \left\{ S \right\})$\;
		\Return{$\{S_i\}_{i=1}^{t}$}\;
\end{algorithm}

\section{Additional Results}
\label{app:k-budget}

Let us present more results --- specifically regarding the class of $ k $-budget functions.
A function $ f $ is \emph{$ k $-budget} $ \equiv $ \[
     f\!\left(S\right) = \min \left\{ k, \absolute{S} \right\}.
\]
Let us denote the distribution over $ k $-budget functions with $ k \in \left\{ 1, \ldots, n \right\} $ chosen uniformly at random as $ \kbudget $.

Since these functions are submodular, we can use the Cohavi--Dobzinski Sketching Algorithm (CDSA)~\cite{Cohavi2017} for comparison regarding $ \alpha $-sketching.
See \Cref{tab:sampling_comparison_kbud} for the results.
Note that we keep the notation consistent with that in \Cref{tab:sampling_comparison} in \Cref{sec: experiments}.

\begin{table}[h!]
    \caption{
    Approximation $\alpha$ minimizing multiplicative error on \covg{} of CDSA \( (f_A, \overline{f}_{\mathcal{K}_A}) \), revealed queried values \( (\underline{f}_{\mathcal{K}_A}, \overline{f}_{\mathcal{K}_A}) \), and \offlinegreedy{} \((\underline{f}_{\mathcal{K}_O}, \overline{f}_{\mathcal{K}_O}) \) with the same budget.
    }
    \label{tab:sampling_comparison_kbud}
    \centering
    \begin{tabular}{lccccc}
        \toprule
        Distribution $ \valueDistribution $ & \# Queried & $ \alpha(f_A, \overline f_{\K_A}) $ & $ \alpha(\underline f_{\K_A}, \overline f_{\K_A}) $ & $ \alpha(\underline f_{\K_O}, \overline f_{\K_O}) $ \\
        \midrule
        \kbudget[5]  & 3.72  & 2.58   & 2.02   & \textbf{1.88}   \\
        \kbudget[8]  & 6.88  & 4.19   & 2.86   & \textbf{2.47}   \\
        \kbudget[10] & 8.91  & 5.18 & 3.33   & \textbf{2.90}   \\
        \bottomrule
    \end{tabular}
\end{table}

\begin{figure*}[t!]
\centering
\begin{tikzpicture}
\begin{groupplot}[
    group style={group size=2 by 1,
    ylabels at=edge left,
    xlabels at=edge bottom,
	horizontal sep=\exploitabilityplotsep,
    },
	xlabel={Steps},
	ylabel={Divergence},
	enlargelimits=false,
	xmin=0,
	ymin=0,
    height=.4\textwidth,
    width=.5\textwidth,
]
    \nextgroupplot[legend to name=sharedlegend, legend columns=-1, title={\kbudget[5]}]
	\exploitabilityplot{figures/domains/k_budget_generator_5.txt}{random_eval}{\textsc{Random}}
	\exploitabilityplot{figures/domains/k_budget_generator_5.txt}{eval}{\textsc{PPO}}
	\exploitabilityplot{figures/domains/k_budget_generator_5.txt}{expected_best_states}{\textsc{Offline Optimal}}
	\exploitabilityplot{figures/domains/k_budget_generator_5.txt}{expected_greedy}{\textsc{Offline Greedy}}
	
    \nextgroupplot[title={\kbudget[10]}]
	\exploitabilityplot{figures/domains/k_budget_generator_10.txt}{random_eval}{\textsc{Random}}
	\exploitabilityplot{figures/domains/k_budget_generator_10.txt}{expected_greedy}{\textsc{Offline Greedy}}
	\exploitabilityplot{figures/domains/k_budget_generator_10.txt}{eval}{\textsc{PPO}}
    \legend{}
\end{groupplot}
\node[anchor=north] at (current bounding box.south) {\pgfplotslegendfromname{sharedlegend}};
\end{tikzpicture}%
	\caption{
	    Comparison of divergence across algorithmic steps for various algorithms, showcasing the \kbudget{} distribution for $n\in\{5,10\}$.
	}
	\label{fig:domain_kbudget}
\end{figure*}
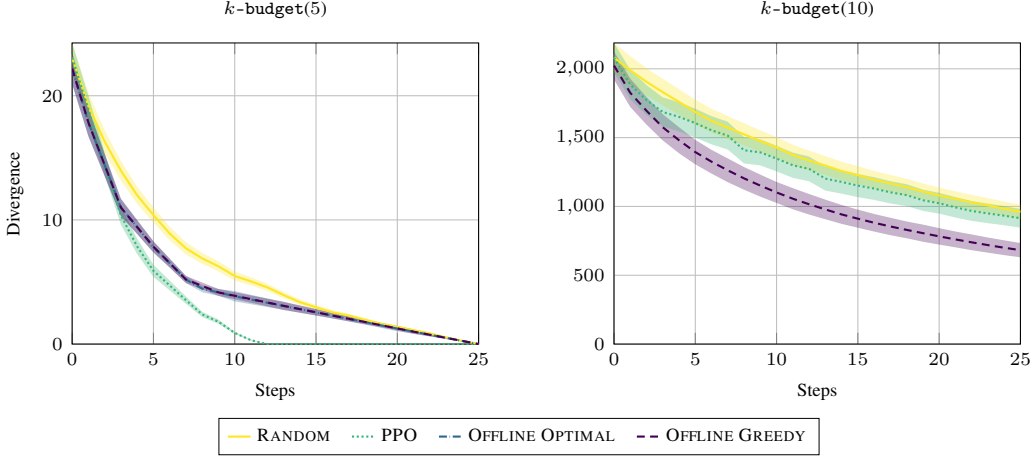

We further offer the results of all our approaches similar to \Cref{fig:domain_main} in \Cref{fig:domain_kbudget}.

\section{Normalization}\label{app:value-normalization}

In our application, it is convenient to transform a set function $f$ by an affine mapping such that after the transformation, the values of the singletons are equal to 0 and the value of the ground set is equal to $\pm1$.
After such transformation, the minimal information $\K_0$ is trivial.

Formally, let $\alpha > 0$ and $\beta \in \mathbb{R}^n$ such that
\begin{equation}\label{eq:SE-transform}
	g(S) = \alpha \cdot f(S) + \sum_{i \in S}\beta_i.
\end{equation}
By considering
\begin{equation*}
	\beta_i = \frac{f(\{i\})}{f(N) - \sum_{i \in N}f(\{i\})}
	\hspace{3ex}\text{and}\hspace{3ex}
	\alpha = \frac{1}{f(N)},
\end{equation*}
we achieve the desired transformation. Proposition~\ref{prop: gap in K} describes the effect on the divergence when such a transformation is applied.

\begin{observation}\label{lem:gap-is-SE}
	The \S-divergence $\gap$ satisfies
	\begin{equation}
		\gap_{\alpha f + \beta}(\K) = \alpha \cdot \gap_f(\K)
	\end{equation}
	where $\alpha > 0$,  $(\alpha v + \beta)(S) \coloneqq \alpha v(S) + \sum_{i \in N}\beta_i$ and $\beta_i \in \mathbb{R}$ for $i \in N$.
\end{observation}

\begin{proof}
	It is a standard result~\cite{Peleg2007} that $f \in \S \implies \alpha f + \beta \in \S$.
	From these two results, we have
	\begin{align*}
		\gap_{\alpha f + \beta}(\K)
		 & =
		\norm{\overline{\alpha f + \beta}_{\K_0 \cup \K} - \underline{\alpha f + \beta}_{\K_0 \cup \K}} \\
		 & =
		\alpha \norm{\overline{f}_{\K_0 \cup \K} - \underline{f}_{\K_0 \cup \K}}                     \\
		 & =
		\alpha\cdot \gap_f(\K)	.
	\end{align*}
\end{proof}

\end{document}